\documentclass{article} 
\usepackage{iclr2026_conference,times}

\usepackage{amsmath,amssymb,amsthm} 
\usepackage{mathtools}              
\usepackage{xcolor}
\usepackage{bm}   
\newtheorem{theorem}{Theorem}
\newtheorem{lemma}{Lemma}
\newtheorem{corollary}{Corollary}
\newtheorem{definition}{Definition}
\theoremstyle{remark}
\newtheorem*{remark}{Remark}
\newtheorem{claim}{Claim}[theorem]
\theoremstyle{definition}
\newtheorem{assumption}{Assumption}
\newtheorem{conjecture}{Conjecture}

\usepackage{amsmath,amsfonts,bm}

\def\eqref#1{equation~\ref{#1}}

\def\1{\bm{1}}

\DeclareMathAlphabet{\mathsfit}{\encodingdefault}{\sfdefault}{m}{sl}
\SetMathAlphabet{\mathsfit}{bold}{\encodingdefault}{\sfdefault}{bx}{n}

\newcommand{\R}{\mathbb{R}}

\usepackage{hyperref}
\usepackage{url}
\usepackage{float}  
\usepackage{subcaption}
\usepackage{enumitem}

\title{The Effect of Attention Head Count on Transformer Approximation}

\author{Penghao Yu\\
Department of Mathematics\\
National University of Singapore\\
\texttt{penghaoyu@u.nus.edu} \\
\And
Haotian Jiang \\
  Institute for Functional Intelligent Materials \\
  National University of Singapore \\
  \texttt{haotian@nus.edu.sg} \\
\And
Zeyu Bao \\
Department of Mathematics\\
National University of Singapore\\
\texttt{zeyu@u.nus.edu} \\
\And 
Ruoxi Yu\\ 
Center for Data Science\\ 
Peking University\\ 
\texttt{yuruoxi@stu.pku.edu.cn}\\
\And
Qianxiao Li \\
  Department of Mathematics \\
  Institute for Functional Intelligent Materials \\
  National University of Singapore \\
  \texttt{qianxiao@nus.edu.sg} 
}

\iclrfinalcopy 
\begin{document}

\maketitle

\begin{abstract}
Transformer has become the dominant architecture for sequence modeling, 
yet a detailed understanding of how its structural parameters influence expressive power remains limited. 
In this work, we study the approximation properties of transformers, 
with particular emphasis on the role of the number of attention heads. 
Our analysis begins with the introduction of a generalized $D$-retrieval task, 
which we prove to be dense in the space of continuous functions, 
thereby providing the basis for our theoretical framework. 
We then establish both upper and lower bounds on the parameter complexity required for $\epsilon$-approximation. 
Specifically, we show that transformers with sufficiently many heads admit efficient approximation, 
whereas with too few heads, the number of parameters must scale at least as $O(1/\epsilon^{cT})$, 
for some constant $c$ and sequence length $T$. 
To the best of our knowledge, this constitutes the first rigorous lower bound of this type 
in a nonlinear and practically relevant setting. 
We further examine the single-head case and demonstrate that 
an embedding dimension of order $O(T)$ allows complete memorization of the input, 
where approximation is entirely achieved by the feed-forward block. Finally, we validate our theoretical findings with experiments on both synthetic data 
and real-world tasks, illustrating the practical relevance of our results.

\end{abstract}

\section{Introduction}

The transformer architecture~\citep{vaswani2017.AttentionAllYou} has become the foundation of modern sequence modeling, driving progress in natural language processing~\citep{devlin2019.BERTPretrainingDeep,brown2020.LanguageModelsAre}, computer vision~\citep{dosovitskiy2020.ImageWorth16x16}, and multi-modal learning~\citep{radford2021.LearningTransferableVisual}. 
Its ability to scale has enabled breakthroughs such as BERT, GPT, and ViT, making it the dominant paradigm across domains. 
Despite this remarkable empirical success, the theoretical principles underlying transformer expressivity remain incomplete. 
In particular, while universal approximation results establish that transformers can approximate arbitrary sequence-to-sequence mappings~\citep{yun2020.AreTransformersUniversal,perez2021.AttentionTuringComplete}, much less is known about how their structural hyperparameters influence approximation efficiency.  

Among transformer hyperparameters, the number of attention heads plays a central role. 
In practice, large models often adopt head counts such as 32, 64, or 128 (e.g., \cite{devlin2019.BERTPretrainingDeep, dosovitskiy2020.ImageWorth16x16, touvron2023.LLaMAOpenEfficient,jiang2023.Mistral7B,grattafiori2024.Llama3Herd} see Table~\ref{table: head count} for more ), 
yet this choice is largely heuristic: there is no principled understanding of how many heads are needed for a given task, 
nor of the costs incurred when the head count is insufficient. 
Theoretical progress on this question has so far been limited. 
Most existing results focus on upper bounds, showing that transformers with sufficiently many heads or with extremely large embedding dimension in the 
single-head case can achieve universal approximation or good approximation rate, but offering little insight into the limitations 
that arise when the head count is insufficient.
Moreover, many analyses rely on strong simplifications—such as restricting to linear embeddings, isolating the attention block, or linearizing the architecture. 
While these assumptions make the problem more tractable, they severely restrict the model’s expressive power 
and prevent the derivation of rigorous lower bounds in realistic nonlinear settings.

In this work, we address this gap by analyzing single-layer transformers on sequence-to-vector tasks. 
To this end, we introduce a new function class, the \emph{generalized $D$-retrieval tasks}, which we design as a structured but expressive family motivated by retrieval problems. 
Each coordinate is defined by
$\bar{z}_i(X_T) = \min_{t \in S_i} f_i(x(t)), i=1, \dots, D$ for subsets $S_i \subseteq [T]$, 
and the overall target takes the form $H(X_T) = F_0(\bar{z}_1(X_T),\dots,\bar{z}_D(X_T))$. 
By construction, this class extends retrieval-style problems while being dense in the space of continuous sequence-to-vector mappings, 
ensuring that results obtained in this setting reflect general approximation behavior.  

The central challenge arises when the number of heads $h$ is smaller than the intrinsic dimension $D$ of the target. 
In this case, multiple coordinates $z_i(X_T)$ must be represented by the same head, creating an information bottleneck: 
the attention layer maps distinct sequences to nearly indistinguishable representations, 
forcing the feed-forward network to perform the separation.
We show that overcoming this bottleneck requires parameter growth exponential in the sequence length $T$, 
namely $O(1/\epsilon^{cT})$ parameters for $\epsilon$-accuracy, 
thus establishing the first rigorous lower bounds for transformers in nonlinear settings. 
In contrast, when $h \geq D$, heads can specialize to distinct coordinates $z_i$, 
eliminating the bottleneck and enabling efficient approximation.

Our results advance the theoretical understanding of attention by showing, that insufficient head count provably limits expressivity in realistic regimes. 
Experiments on both synthetic tasks and real-world retrieval data confirm that the predicted scaling laws persist in practice. 

\paragraph{Contributions.}  
Our main contributions are as follows: 

First, we establish the first rigorous lower bounds for transformers in nonlinear settings, showing that when $h<D$, parameter complexity grows exponentially with sequence length. 

Second, we provide constructive upper bounds, proving that $h \geq D$ enables efficient approximation with parameter growth independent of sequence length $T$.

Third, in the memorization regime, single-head transformers with embedding dimension $n \geq Td$ approximate by memorizing sequences, with the complexity residing in the feed-forward block.

\section{Related Work}
Several works have studied the approximation and expressivity properties of transformers. 
The universal approximation property was first established in \cite{yun2020.AreTransformersUniversal}, 
and later extended to transformers with sparse attention matrices in \cite{yun2020.ConnectionsAreExpressive}. 
The approximation rate of single-layer transformers with one head was analyzed in \cite{jiang2024.ApproximationRateTransformer}.
\cite{amsel2024.QualityQuantityAttention} investigated how the rank of the attention matrix influences expressivity for a specific nearest-neighbor target can be constructed. 
They showed that when the rank is insufficient, the number of heads required for approximation grows exponentially, independent of sequence length.
In a related direction, \cite{bhojanapalli2020.LowRankBottleneckMultihead} argued that setting the rank of the attention matrix equal to the sequence length enhances expressivity.
 Beyond finite-dimensional settings,
\cite{takakura2023.ApproximationEstimationAbility} considered sequences of infinite dimension, 
characterizing approximation rates in terms of target function smoothness.
Similarly, 
\cite{wang2024.UnderstandingExpressivePower} studied special classes of target functions and demonstrated that approximation error scales polynomially with the number of heads. 
In addition to these approximation-theoretic results,
several works have investigated broader notions of expressivity. 
\cite{dehghani2019.UniversalTransformers,perez2021.AttentionTuringComplete} established the Turing completeness of transformers, 
and \cite{giannou2023.LoopedTransformersProgrammable} showed that transformers can represent arbitrary computer programs in a looped setting. Finally, 
\cite{mahdavi2024.MemorizationCapacityMultiHead} examined memorization capacity, proving that the number of samples that can be stored scales linearly with the number of heads.

\section{Preliminaries}

\paragraph{Input and Output}
We consider the input space
\begin{equation}
    \mathcal{X}_T
    = \bigl\{\, x(s) \in [0,1]^d \;:\; s \in [T] \,\bigr\},
    \label{eq:input-space}
\end{equation}
where $[T] = \{1,\dots,T\}$. We call $T$ the length of the input sequence.  
The output is a single vector $y \in \mathbb{R}^l$, where $l$ is independent of $T$ and specified by the task.
\par For example, in a text retrieval task one may take $d$ to be the max number of tokens per candidate, $T$ the number of candidates, and $l$ the size of the output representation.

\paragraph{Input Representation.}
Each token is mapped to an $E$-dimensional vector by a trainable encoder
\[
    P_{\phi} : [0,1]^d \times [T] \to \mathbb{R}^E,\qquad
    (\bm{x},s) \mapsto P_{\phi}(\bm{x},s),
\]
which jointly incorporates the content $\bm{x}$ and its position $s$.
Given $X_T=\{x(s)\}_{s=1}^T \in \mathcal{X}_T$, the embedded sequence is
\begin{equation}
    \hat{X}_T
    = \{\, \hat{x}(s) = P_{\phi}(x(s),s) \in \mathbb{R}^E : s \in [T] \,\}.
    \label{eq:pos-embed}
\end{equation}
This formulation subsumes common designs where $P_{\phi}$ combines a content embedding
with either learned or deterministic positional encoding.
\par For example, if $\text{Emb}(x)$ is a content embedding map and $p(t)$ a positional code, then common schemes correspond to \[ P_\phi(x(t),t) = \text{Emb}(x(t)) + p(t) \quad\text{(additive PE)}, \] or \[ P_\phi(x(t),t) = \bigl(\text{Emb}(x(t)),\,p(t)\bigr) \quad\text{(concatenated PE)}. \]
Following common practice, we append a trainable \emph{classification token} $\hat{c}_0 \in \mathbb{R}^E$ to the sequence.  
The final input to the transformer is
\begin{equation}
    \hat{X}[T] = \{\hat{x}(1),\dots,\hat{x}(T),\hat{c}_0\} \in \mathbb{R}^{E \times (T+1)} ,
\end{equation}
and the output $\hat{y}$ is taken from the $(T\!+\!1)$-th position corresponding to $\hat{c}_0$.

\paragraph{Transformer Hypothesis Class}
\par With the input space and embedding defined, we then  formulate the transformer hypothesis space.  
\par We consider a single-layer transformer based on the standard architecture \citep{vaswani2017.AttentionAllYou}, with two modifications introduced for analytical simplicity. Firstly, we omit layer normalization to simplify the analysis, while acknowledging its practical importance, and we conjecture that our key lower bound (Theorem~\ref{thm2:p2}) still holds when layer normalization is present. Secondly, we exclude residual connections outside the feed-forward network. In the single-layer sequence-to-vector setting, where the output is read from a designated classification token, the residual branch can be merged into the feed-forward transformation by reparameterization, thus these likewise do not alter the expressive power of the architecture.

For an $h$-head, single-layer transformer, let $n$ denote the embedding
dimension \emph{per head} and $E=nh$ the total embedding dimension.  
The output is
\begin{equation}
\begin{aligned}
    \hat{y}
    &= \hat{H}(\hat{X}[T]) 
    = \hat{F}\!\Bigl(
        \hat{c}_0 +
        W_O \,\operatorname{Concat}_{i=1}^{h}
        \Bigl(
            \sum_{t=1}^{T}
            \sigma\!\left[
                (W_{Q,i}\hat{c}_0)^\top W_{K,i}\hat{x}(t)
            \right]
            W_{V,i}\hat{x}(t)
        \Bigr)
    \Bigr),
    \label{eq:transformer}
\end{aligned}
\end{equation}
where for each head $i$,
$W_{Q,i}, W_{K,i} \in \mathbb{R}^{n \times E}$ are the query/key projection matrices,  
$W_{V,i} \in \mathbb{R}^{n \times E}$ is the value projection,  
$W_O \in \mathbb{R}^{E \times E}$ is the output projection applied to the concatenated heads,  
and $\hat{F}:\mathbb{R}^{E}\to\mathbb{R}^l$ is a feed-forward network which we call it the \emph{feed-forward block}.  
The softmax with scaling factor $\beta$ is defined by
\begin{equation}
    \sigma[\rho](t)
    = \frac{\exp{\!\bigl(\beta\,\rho(t)\bigr)}}
           {\sum_{t'=1}^T \exp{\!\bigl(\beta\,\rho(t')\bigr)}},
    \qquad \beta > 0.
    \label{eq:softmax}
\end{equation}
and $\beta>0$ may be chosen arbitrarily large in order to make the softmax attention mechanism approximate a hardmax.

We denote this family by
\begin{equation}
    \mathcal{H}(h,n,d,T,M),
\end{equation}
the class of single-layer transformers with $h$ heads, per-head embedding dimension $n$, input dimension $d$, sequence length $T$, and parameter count $M$.  
Each $H \in \mathcal{H}(h,n,d,T,M)$ is a mapping
\[
    H: \mathbb{R}^{d\times T} \to \mathbb{R}^l ,
\]
implemented by the encoder $P_\phi:[0,1]^d\times[T]\to\mathbb{R}^{nh}$, concatenation of the classification token $\hat{c}_0$, a multi-head attention layer with projections $\{W_{Q,i},W_{K,i},W_{V,i}\}_{i=1}^h, W_O$, and a feed-forward network $\hat{F}:\mathbb{R}^{nh}\to\mathbb{R}^l$. Thus $H$ has the form~\eqref{eq:transformer}, 
with parameter count $k$ referring to the weights and biases in FFNs $(P_\phi,\hat{F})$.

\paragraph{Approximation Problem}
With the hypothesis class specified, we now formalize the approximation problem, which provides the framework for analyzing the expressive power of transformers.
\begin{definition}[$\epsilon$-approximation]
    Let $\mathcal{X}_T \subset \mathbb{R}^{d \times T}$ be a compact domain, and let 
    $F: \mathcal{X}_T \to \mathbb{R}^l$ be a target function. 
    We say that the hypothesis class $\mathcal{H}(h,n,d,T,M)$ $\epsilon$-approximates $F$ on $\mathcal{X}_T$ 
    if there exists $\hat{H} \in \mathcal{H}(h,n,d,T,M)$ such that
    \[
        \sup_{X_T \in \mathcal{X}_T} \;\|\hat{H}(X_T) - F(X_T)\|_\infty < \epsilon.
    \]
\end{definition}


\section{Generalized $D$-Retrieval Tasks}

\paragraph{Target functions.}\label{sec:TargetFunctions}  
To motivate our construction, consider a simple one-dimensional example:  
\[
    \mathcal{X}_T = \{\,X_T = (x(1),\dots,x(T)) : x(t)\in[0,1] \,\},
\]
with target
\begin{equation}
    H(X_T) \;=\; \max_{1 \le t \le T} x(t) \;+\; \min_{1 \le t \le T} x(t).
    \label{eq:toy-example}
\end{equation}
This task requires the model to extract two distinct features from the sequence—the maximum
and the minimum—before combining them.  
It can thus be viewed as a retrieval problem with two independent features being aggregated.  

This example illustrates the broader idea behind our target class: 
retrieval-style problems where multiple salient features must be identified and combined. 
We now formalize this intuition by defining the family of \emph{generalized $D$-retrieval tasks}.

\paragraph{Mathematical Formulation}Formally, for each $i=1,\dots,D$, let $f_i:[0,1]^d\to[0,1]$ be $C^2$ and define
\begin{equation}
    \bar{z}_i(X_T) \;=\; \min_{t \in S_i} f_i(x(t)),
    \qquad S_i \subseteq [T], \;\; |S_i|\ge \frac{1}{4} T,
    \label{eq:zi}
\end{equation}
so that $\bar{z}(X_T)=(\bar{z}_1(X_T),\dots,\bar{z}_D(X_T))\in[0,1]^D$.  
The target is then
\begin{equation}
    H(X_T) \;=\; F_0\!\bigl(\bar{z}(X_T)\bigr),
    \label{eq:target}
\end{equation}
where $F_0:[0,1]^D\to\mathbb{R}$ is $C^1$.  
For vector-valued targets $H:[0,1]^{d\times T} \to \mathbb{R}^l$ defined with the same functions $f_i$, subsets $S_i$, and an outer map $F_0 : [0,1]^D \to \mathbb{R}^l$, the extension is applied coordinate-wise, since each coordinate function of $F_0$ can be approximated independently. Therefore, it suffices to consider the scalar-valued case. We denote by $\mathcal{F}_D^{d,T}$ the class of all such functions $H$. 

Related sparse sequence-to-sequence retrieval tasks, such as the $q$-sparse token regression (qSTR) model of \citep{mousavi2025transformers}, where each output position depends on at most $q$ relevant input tokens, can be viewed as sequence-to-sequence analogues of our formulation. Their results on the sample complexity of single-layer Transformers with at least $q$ heads are complementary to our approximation-theoretic guarantees in the generalized $D$-retrieval setting.
\medskip
\paragraph{Assumptions on the target class} For the theoretical analysis to be tractable we impose the following conditions: 
\begin{assumption}[Model constraints]\label{assump:setup} The model constraints are as follows:\\
(1.1) 
The embedding $P_\phi$ satisfies  
\[
    \|P_\phi(x(s),s)\|_2 \leq 1, \qquad \forall\, s \in [T], \; X_T \in \mathcal{X}_T,
\]
ensuring embedded inputs remain uniformly bounded. \\
(1.2) 
The post-attention mapping $\hat{F}$ is a two-layer feed-forward network 
with $1$-Lipschitz activation, hence a universal approximator on compact domains.  \\
(1.3) 
All weights in $\hat{F}$ and entries of the attention matrices 
$\{W_{Q,i},W_{K,i},W_{V,i}\},W_O$ are bounded in magnitude by $1$, ensuring stability of the model.  
\end{assumption}
\begin{assumption}[Target class constraint]\label{assump:target}
The target functions defined in \eqref{eq:target} satisfy the following: \\ 
(2.1) Each $f_i:[0,1]^d \to [0,1]$ attains its unique global minimum $z_i$ at a point $x^{(i)}\in[0,1]^d$.\\ 
(2.2) The minimizers $\{x^{(i)}\}_{i=1}^D$ are pairwise distinct.  \\
(2.3) The Hessian $\nabla_x^2 f_i(x^{(i)})$ is positive definite for all $i=1,\dots,D$.  \\
(2.4) The gradient $\nabla_z F_0(z_1,\dots,z_D)$ has all coordinates strictly nonzero.  
\end{assumption}
\begin{remark}
    Assumption~\ref{assump:target} excludes only degenerate cases while preserving broad generality for both the functions $f_i$ and the outer map $F_0$. More specifically: Assumptions (2.1) and (2.4) ensure that each $f_i$ behaves regularly around its minimizer. A degenerate example excluded by these assumptions is $f_i(x) \equiv c_0$ for constant $c_0$, which is totally independent of the input; Assumption (2.2) requires distinct minimizers, which allows partitioning the space into $D$ disjoint  basins around each minimizer $x^{(i)}$. A degenerate example excluded by this assumption is $f_1 = f_2 = \dots = f_D$; Assumption (2.3) enforces sensitivity of the target to small perturbations near the minimizers, ruling out trivial flat cases (such as $F_0 \equiv c_0$ for constant $c_0$) where no meaningful separation can be made.
\end{remark}

\medskip
Having introduced the generalized $D$-retrieval tasks, it remains to ask whether this
family is sufficiently expressive. 
To address this, we now establish the
\emph{universality of the target class}: the family is dense in the space of continuous
sequence-to-vector mappings. 
\begin{theorem}[Density of the target class]\label{thm:thm1}
    For fixed $d,T$, the family $\{\mathcal{F}_D\}_{D=1}^{\infty}$ is dense in 
    $C(\mathcal{X}_T)$. That is, for every $F \in C(\mathcal{X}_T)$ and every 
    $\epsilon>0$, there exists $D$ and $f \in \mathcal{F}_D$ such that
    \[
        \max_{X \in \mathcal{X}_T} |F(X) - f(X)| \le \epsilon.
    \]
\end{theorem}
\noindent
The proof is deferred to Appendix~\ref{app:proof-thm1}. 
This density property highlights that our specially designed
target family is not overly restrictive; rather, it forms a sufficiently general
class to capture arbitrary continuous sequence-to-vector mappings.
Beyond density, we show that $D$ is indeed the \emph{intrinsic dimension} of this target, which means
that it is the unique $D \ll T$ for which the target $H$ can be represented in the generalized $D$-retrieval task form.

\begin{corollary}[Uniqueness of intrinsic dimension]\label{cor:uniqueness}
    If task $H$ can be represented by $(\{f_i, S_i\}_{i=1}^{D_1}, F_0)$ and $(\{\tilde{f}_i, \tilde{S}_i\}_{i=1}^{D_2}, \tilde{F}_0)$, satisfying Assumption~\ref{assump:setup} and \ref{assump:target} with $D_1^2+D_2^2 \le \frac{1}{50} T$, then $D_1 = D_2$.  
\end{corollary}
This corollary justifies that $D$ comes from the intrinsic property of the target and is invariant across its form of representation. The proof is deferred to Appendix~\ref{app:proof-thm2}


\section{Approximation rate of Generalized $D$-Retrieval Tasks}
Theorem~\ref{thm:thm1} establishes that the generalized $D$-retrieval tasks form a dense family in the space of continuous sequence-to-vector functions. The next step is to analyze the efficiency with which transformers approximate these functions.  
To this end, we begin by stating two standard approximation assumptions regarding how well the
fundamental building blocks of the target can be approximated.

\begin{assumption}[Approximation of components]\label{assump:approximation}
We assume the following approximation properties hold.\\  
(A1) There exist constants $C_1>0$ and $\gamma>0$ such that for every $\delta>0$, the function 
$F_0:[0,1]^D\to\mathbb{R}$ can be $\delta$-approximated by a two-layer feed-forward network 
$\Phi_\delta$ of width at most $C_1/\delta^\gamma$, i.e.,
\[
    \sup_{x\in[0,1]^D} |F_0(x)-\Phi_\delta(x)| \le \delta.
\]  
(A2) There exist constants $C_2>0$ and $\gamma>0$ (possibly different from (A1)) such that for each 
$i=1,\dots,D$ and every $\delta>0$, the function $f_i:[0,1]^d\to[0,1]$ can be $\delta$-approximated 
by a two-layer feed-forward network $\Psi_{i,\delta}$ of width at most $C_2/\delta^\gamma$, i.e.,
\[
    \sup_{x\in[0,1]^d} |f_i(x)-\Psi_{i,\delta}(x)| \le \delta.
\]  
\end{assumption}
These assumptions are reasonable: by the classical result of \citep{cybenko1989.ApproximationSuperpositionsSigmoidal}, 
two-layer networks can approximate continuous functions on compact domains. 
In particular, if the Barron norm is finite, one may take $\gamma=2$ \citep{barron1993.UniversalApproximationBounds}; 
even in the worst case, setting $\gamma=\max(d,D)$ yields approximation rates comparable 
to uniform grid discretizations, which still suffices for our analysis.

We now present our main theoretical result. 
It establishes upper and lower bounds on the approximation rates of transformers 
within the generalized $D$-retrieval framework. 
In particular, the lower bound in part (2) provides the first rigorous evidence that insufficient head count $h < D$ leads to exponential parameter complexity, 
revealing a fundamental expressivity bottleneck.
\begin{theorem}[Approximation rates of transformers] \label{thm:thm2}
Fix $d,T$. Under Assumption~\ref{assump:approximation}, the following hold 
for any target $H \in \mathcal{F}_D^{d,T}$:

\begin{enumerate}[label=(\arabic*),ref=\thetheorem~(\arabic*)]
    \item \textbf{Sufficient expressivity with $D$ heads.} \label{thm2:p1} 
    For $h=D$ and embedding dimension $n=2$ per head, there exists a constant $C_{d, D, T}>0$ such that $\forall M > \frac{C_{d, D, T}}{\epsilon^\gamma}$. the hypothesis class 
    $\mathcal{H}(h,n,d,T,M)$ $\epsilon$-approximates $H$.

    \item \textbf{Lower bound with $s<D$ heads.}  \label{thm2:p2}
    For $h=s<D$, define
    \[
        k=\frac{(\frac{1}{4} T-s -D +1)}{(n+1)s+1}-1 ,
    \]
    then 
    \[
    \min \bigl\{\, M \;:\; \mathcal{H}(h,n,d,T,M)\ \epsilon\text{-approximates } 
    H \,\bigr\} \;=\; \Omega\!\left(\tfrac{1}{\epsilon^k}\right).
\]

    \item \textbf{Single-head large embedding dimension.}  \label{thm2:p3}
    For $h=1$ and per-head embedding dimension $n\ge Td$, if the feed-forward block is a $5$-layer ReLU neural network, then there exists a constant $C_{d,D, T}>0$ such that 
    for all $M > \tfrac{C_{d,D,T}}{\epsilon^{1+\gamma}}$, the hypothesis class 
    $\mathcal{H}(h,n,d,T,M)$ can $\epsilon$-approximate $H$.
\end{enumerate}
\end{theorem}

\begin{remark}
    We clarify the precise role of the assumptions and constants in Theorem~\ref{thm:thm2}. \\
    (1)Theorems~\ref{thm2:p1} and~\ref{thm2:p3} require only Assumption~\ref{assump:approximation}, 
    while Theorem~\ref{thm2:p2} relies only on Assumptions~\ref{assump:setup} and~\ref{assump:target}.\\
    (2) The constant in Theorem~\ref{thm:thm2} can be made explicit as $C_{d,D,T} \;=\; C_{d,D}\,(rT)^{-\alpha_{d, D} T}$,where $r>0$ is determined by the local behavior of the functions $f_i$ around $x^{(i)}$ and of $F_0$, 
    and $\alpha_{d, D}$ depends only on $d$ and $D$. This form is valid in the regime $d, D \ll T \ll 1/\epsilon$.\\
    (3) The exponent coefficients in Theorems~\ref{thm2:p1} and~\ref{thm2:p3} differ because, in 
    Theorem~\ref{thm2:p3}, the network $\hat{F}$ also needs $\Omega(T/\epsilon)$ parameters to approximate 
    the ``max-like'' operation. This yields a bound of the form $M \;\le\; \frac{1}{\epsilon^{\,\max(1,\gamma)}}$,
    and for notational simplicity we write $M \le 1/\epsilon^{1+\gamma}$.\\
    We provide the detailed proof in Appendix~\ref{app:proof-thm2}. We also justify the tightness of Theorem~\ref{thm2:p2} in Appendix~\ref{rem:lower-bounds}.
\end{remark}

Theorem~\ref{thm:thm2} highlights how approximation efficiency depends on head count: 
enough heads allow specialization, too few force inefficient compression, and 
a single large head can rely on memorization.  
To illustrate these cases concretely, we now revisit the toy example from equation~\ref{eq:toy-example} and discuss how each part of the theorem works in that setting.

\paragraph{Case (1): $h \geq D$ heads.}  
Theorem~\ref{thm2:p1} shows that when the number of heads matches the intrinsic dimension $D$ of the target, 
the transformer can allocate one head per component feature, allowing each head to specialize 
and leaving the feed-forward block to aggregate their outputs.  
This yields efficient approximation with $O(M^{-1/\gamma})$ error for parameter count $M$, 
independent of sequence length $T$.  

In the toy example with $D=2$, one head naturally tracks the maximum and the other the minimum, 
so the task is solved directly without incurring inefficiency.  
This illustrates how having “enough heads” removes the unfavorable scaling in $T$ and explains 
the practical advantage of multiple heads beyond universal approximation results 
(e.g.,~\citet{kajitsuka2023.AreTransformersOne}).

\paragraph{Case (2): $h < D$ heads.}  
Theorem~\ref{thm2:p2} establishes that when the number of heads is 
smaller than the intrinsic dimension $D$, the parameter count required to achieve 
a given accuracy can grow exponentially in the sequence length $T$.  
This lower bound highlights why insufficient heads lead to severe inefficiency.  
Intuitively, each head can be viewed as specializing in one coordinate of the 
minima structure in~\eqref{eq:target}.  
When $h < D$, a single head must encode multiple roles simultaneously.  

In the toy example with $D=2$, one head is forced to capture both the maximum 
and the minimum across all $T$ positions.  
Since softmax attention only produces weighted averages, the head must effectively 
encode information from multiple sequence elements simultaneously.  
As $T$ increases, the number of relevant elements to distinguish grows linearly 
with $T$, yet they are compressed into an $n$-dimensional vector whose size does not scale with $T$. 
The feed-forward block must then disentangle these increasingly entangled 
representations, which requires parameters exponential in $T$.  
This explains why the parameter requirement scales as $\Omega(1/\epsilon^{cT})$ 
and why the scaling improves dramatically once $h \geq D$.

Theorem~\ref{thm2:p2} is proved with the following idea: (i) each head selects its most responsive locations $(y_j,t_j)$ in $D$ disjoint minima basins around $x^{(i)}$; (ii) because $s<D$, there is at least one segment $G_i \subset B(x^{(i)},r)$ that no head focuses on. We then consider the segment $G_i$ in it; (iii) along this segment (suppose it is $G_1$), we construct many candidate subsequences and, by a pigeonhole argument, obtain two subsequences $Z_1,Z_2$ whose post-attention representations are almost identical but whose contribution to $f_1$ differs; (iv) these subsequences are then embedded into full sequences $W_1,W_2$, which the target function separates by at least $3\epsilon$, while the attention block maps them within $O(\epsilon^{k+1})$, forcing a large feed-forward network. 

The intuition by which we derive the large network is different from geometric arguments in existing works such as \citep{yehudai2019power}. We directly made use of the fact that the network must be either large or have large parameters to approximate a function with large Lipschitz norm. 

\paragraph{Case (3): single head with large embedding.}  
Theorem~\ref{thm2:p3} shows that when the embedding dimension scales with the sequence length, 
$E = n \geq Td$, the model can encode the entire sequence into the classification token $\hat{c}_0$.  
Concretely, each input can be embedded as $e_t \otimes x(t) \in \mathbb{R}^{Td}$, where $e_t$ is the $t$-th standard basis vector, 
so that trivial attention aggregates to 
\[
    \tfrac{1}{T}\,(x(1),\dots,x(T)) \in [0,1]^T,
\]
which preserves the full sequence.  
The feed-forward block $\hat{F}$ can then recover the target relation efficiently.  
Unlike memorization of training data \citep{mahdavi2024.MemorizationCapacityMultiHead}, this mechanism can 
generalize since it captures the relation itself.  
Moreover, approximating extrema functions such as $\max$ and $\min$ with a shallow ReLU network 
is straightforward (see Lemma~\ref{lem:relu-max}), requiring width $O(T/\epsilon)$ for $\epsilon$ accuracy.  
However, this regime is impractical, as it demands embedding dimensions that grow linearly with $T$.

As deeper transformer are more commonly used in practice, here we conjecture the extension of Theorem~\ref{thm2:p2} onto the $L$-layer case. 
\begin{conjecture}[Multilayer transformer case] \label{Conjecture1} A necessary condition for efficient approximation is $L \cdot h \;\ge\; D$. When the head number is insufficient across layers, we conjecture the lower-bound scaling for some constants $a_L, b_L > 0$ depending only on depth $L$ to be
\[
    \log(\mathrm{ParamCount})
    = \Omega\!\Bigl(|\log \epsilon| \cdot \tfrac{a_L \, T^{\, b_L}}{n h} \Bigr),
\]
\end{conjecture}
Experiments on the synthetic dataset in Section~\ref{experiment1} with $2$-layer transformer with no positional encoding and no layer norm has also verified this transition at $D = L \cdot h$. (See Table~\ref{tab:2layer} in Appendix)

\section{Experiments}
Theorem~\ref{thm:thm2} provides theoretical insights into how the approximation 
ability of transformers depends on the number of heads. In this section, we 
illustrate these insights empirically. We begin with synthetic tasks that mirror 
the structure of the generalized $D$-retrieval tasks, and then turn to real 
datasets (MS MARCO and CIFAR-10) to examine whether similar scaling behaviors 
arise in practice.
\subsection{Numerical verification of Theorem~\ref{thm:thm2} with synthetic dataset}\label{experiment1}
 
We design a synthetic task aligned with the target class analyzed in Theorem~\ref{thm:thm2}. 
Given a sequence $X=\{x(1),\dots,x(T)\}$ of length $T$ with $x(t)\in\mathbb{R}^4$, 
the output is
\[
    y = \sum_{i=1}^4 \max_{1 \leq t \leq T} \, a_i^\top x(t),
\]
where $a_1,\dots,a_4 \in \mathbb{R}^4$ are fixed. Inputs are sampled i.i.d.\ 
from $x(t) \sim \mathcal{N}(0,I_4)$. For $T \in \{8,16,32,64,128\}$ we generate 
$8000$ training and $2000$ validation examples.  

On this task, we evaluate single-layer transformers with head numbers $h \in \{1,2,3,4,5\}$ and fixed per-head embedding dimension. 
Each $x(t)$ is embedded via a two-layer ReLU MLP and concatenated with a 
trainable classification token $c_0$, after which a single-layer multi-head 
attention block (without residuals or normalization) processes the sequence. 
A two-layer GELU MLP then outputs the scalar prediction. Both MLPs have the same hidden dimension $N$.

Then for each $(h,T)$, models are trained under multiple random seeds. We report the 
\emph{minimal normalized mean squared error (NMSE)} across seeds to reduce 
optimization noise and highlight expressivity. NMSE, equivalent to 
$1-R^2$, corrects for the variance shrinkage of maxima as $T$ grows, thus 
enabling fair comparison across lengths. Further training details and explanations are given in 
Appendix~\ref{app:exp1-training-details}.


\begin{figure*}[h]
    \centering
    \begin{subfigure}{0.35\textwidth}
        \centering
        \includegraphics[width=\linewidth]{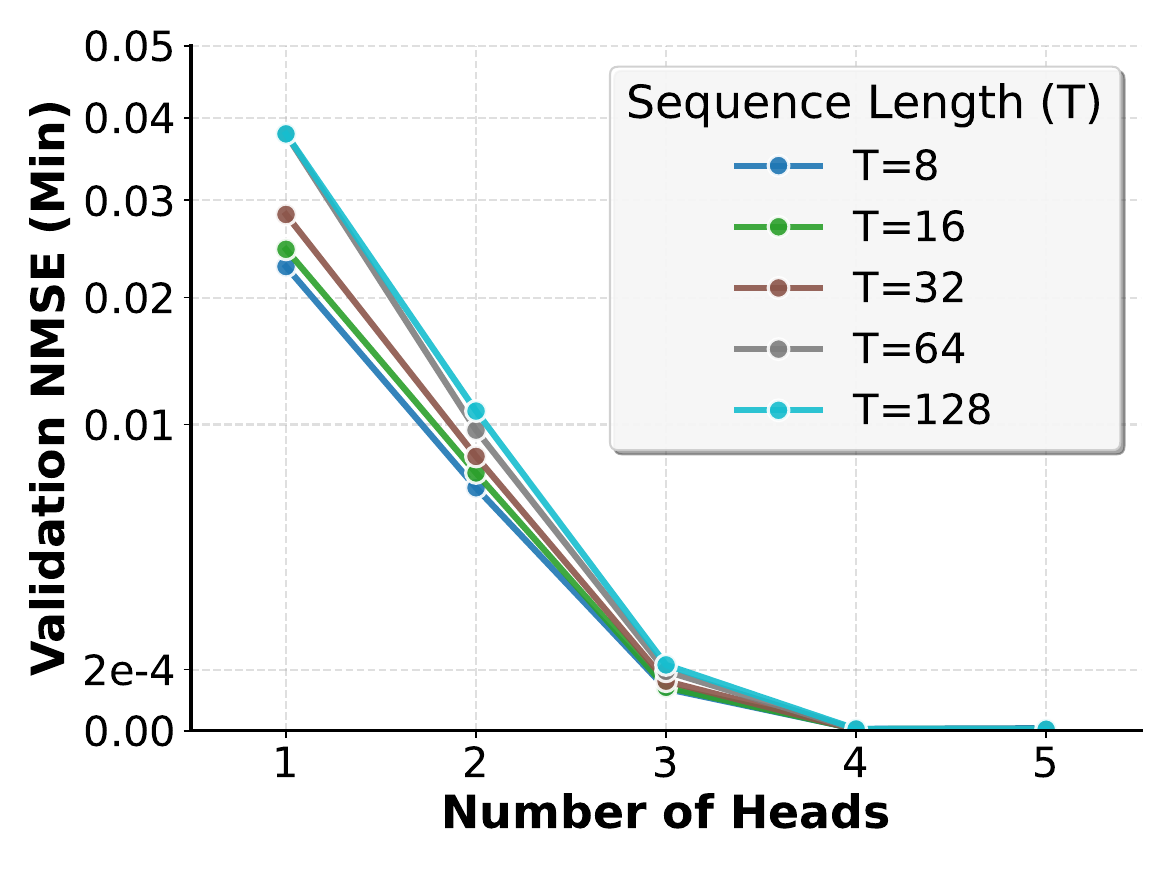}
        \caption{\textbf{NMSE vs.\ Number of Heads $h$.}}
        \label{fig:synthetic-nmse-a}
    \end{subfigure}
    \hfill
    \begin{subfigure}{0.63\textwidth}
        \centering
        \includegraphics[width=\linewidth]{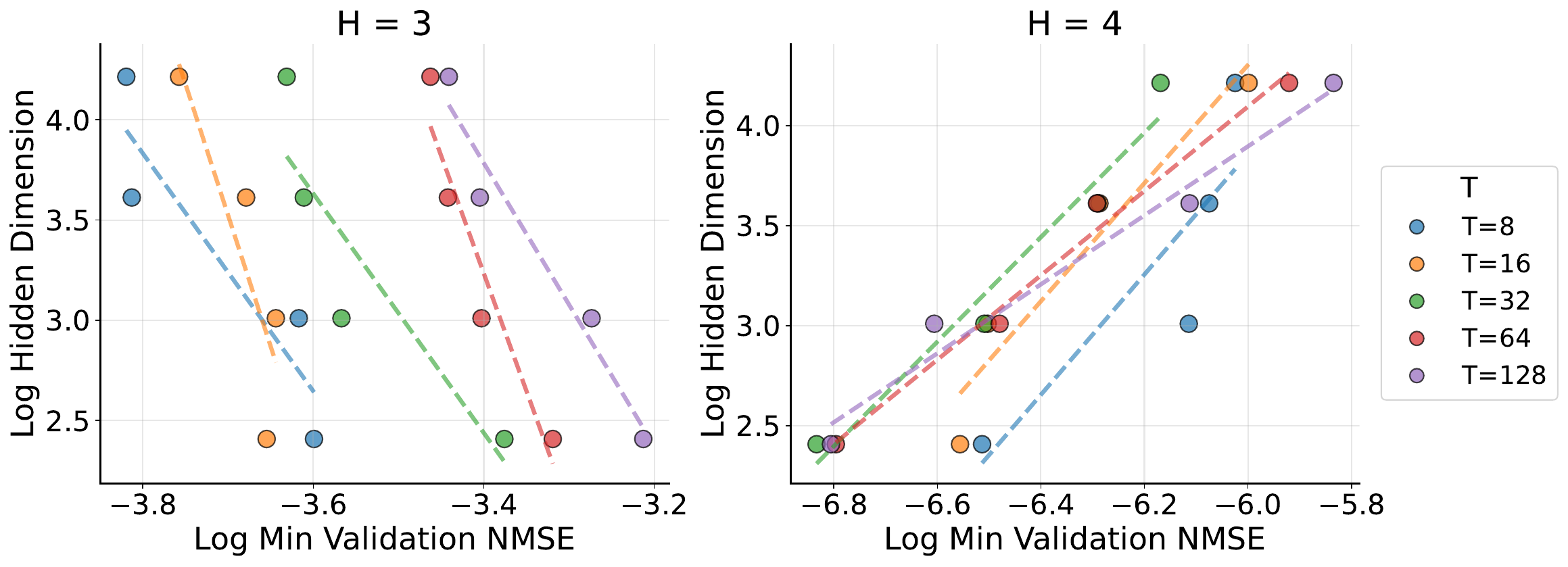}
        \caption{\textbf{Log $N$ vs. Log Accuracy (NMSE)}}
        \label{fig:synthetic-nmse-b}
    \end{subfigure}

    \caption{Results on the synthetic example.  
    (a) NMSE vs.\ number of heads $h$ for sequence lengths $T \in \{8,16,32,64,128\}$, hidden dimension fixed at $N=32$. Note that there is a transition at $h=4$. (A table of mean and variance values corresponding to these curves 
is provided in Table~\ref{table:synthetic-variance}.)  
    (b) Log Hidden Dimension $N$ vs. Log Accuracy for different sequence lengths $T$. The parameter count $k$ for the MLPs change linearly with $N$. (Plots for $h=1$ and $h=2$ are in Figure \ref{fig:synthetic-nmse-appendix}.) }

    \label{fig:synthetic-nmse} 
\end{figure*}


Figure~\ref{fig:synthetic-nmse-a} shows minimal validation NMSE versus head number $h$ across sequence lengths $T$. 
Performance improves monotonically with $h$ and exhibits a clear transition near the intrinsic dimension $D=4$. 
For $h<D$, NMSE grows with $T$, as limited heads must encode multiple extrema and the FFN becomes inefficient. 
Once $h \ge D$, curves flatten across $T$, indicating that heads specialize to different coordinates and the FFN aggregates them very effectively. 
Normalization by NMSE ensures comparability across $T$, despite the increasing concentration of the max-of-Gaussians target.

Figure~\ref{fig:synthetic-nmse-b} highlights a phase transition between $h=3$ and $h=4$, 
with $h=D=4$ equal to the intrinsic dimension of the target. 
When $h\le 3$, the negative log NMSE scales approximately linearly with the log parameter 
count (proportional to the MLP hidden dimension $N$), in agreement with 
Theorem~\ref{thm2:p2}. Moreover, for a fixed parameter count, larger $T$ yields higher NMSE (worse 
approximation). Equivalently, as indicated by the fitted scaling lines, achieving 
the same error requires larger parameter counts when $T$ increases, in line with 
Theorem~\ref{thm2:p2}. In contrast, for $h=4$ these trends change qualitatively. 
Validation error reaches the order of $10^{-6}$, indicating near-perfect generalization, 
yet the slope with respect to parameter count reverses: larger MLPs yield slightly higher 
validation NMSE, a signature of tiny overfitting. 
The dependence on $T$ also changes in this regime; see 
Remark~\ref{observation:reverse} in Appendix for details. 

We also conducted experiments on synthetic data with the widely used scheme of fixing $E=nh=32$ constant(For $h=3,5$, we choose per-head embedding dimension to be $\lceil 32/h\rceil$, and total embedding dimension becomes $E=33, 35$. See Table~\ref{tab:heads-fixed-E} in Appendix for details), as well as experiments on synthetic datasets with $D=3$ (See Table~\ref{tab:D3-results} in Appendix for details). Both of the above experiments demonstrate similar trends to the $D=4$ experiments, implying the robustness of our results.

\subsection{Experiments on Real Datasets}\label{sec:real-exp}
We conduct two additional experiments on real datasets to assess the 
practical relevance of our theoretical findings. 
The first is a text retrieval task based on MS MARCO, and the second is 
an image classification task based on CIFAR-10. As there is no natural NMSE-like metric on retrieval tasks and accuracy is most widely used, we focus on training accuracy to isolate architectural expressivity from issues related to optimization or data scarcity. For completeness, we also report test accuracy for both experiments in Table~\ref{tab:MS_val_mean_var} and \ref{table:vit_val_accuracy} in Appendix. The experiments examine whether the phase transition around the intrinsic 
dimension $D$, predicted by Theorem~\ref{thm:thm2}, also manifests in practice.

\paragraph{MS MARCO (text retrieval).}  
We construct retrieval-style datasets from the MS MARCO passage ranking collection 
\citep{bajaj2018.MSMARCOHuman}, where each query is paired with one positive passage and 
$T-1$ mined hard negatives ($T \in \{8,16,32,64\}$). 
We train a two-layer transformer encoder with per-head embedding dimension fixed at $32$, 
varying the number of heads across $\{1,2,4,6,8,10,12,14,16\}$. 
Input text is tokenized using the BERT tokenizer, and word, positional, and segment 
embeddings from pretrained BERT \citep{devlin2019.BERTPretrainingDeep} are kept frozen. 
These 768-dimensional embeddings are linearly projected to the embedding size 
$E = \text{heads} \times 32$, after which only the projection and transformer layers are trained. 
Full dataset construction and training details are given in Appendix~\ref{app:exp-msmarco}.

\paragraph{CIFAR-10 (image classification).}    
We further evaluate on the CIFAR-10 dataset \citep{krizhevsky2009.LearningMultipleLayers} using a 
four-layer Vision transformer (ViT) \citep{dosovitskiy2020.ImageWorth16x16}. 
Each image of size $32 \times 32$ is divided into non-overlapping $8 \times 8$ patches 
(patch size $=8$), which are linearly embedded. The transformer encoder has per-head 
embedding dimension fixed at $16$, and we vary the number of heads across 
$h \in \{1,2,4,8,10,11,12,13,14,16,20,24\}$. 
To vary the sequence length, we extend the border with interpolation around each 
image to enlarge its side length, after which the sequence length is given by 
$\left(\tfrac{\text{image side length}}{\text{patch size}}\right)^2 + 1$, including the classification token. 
Figure \ref{fig:cifar padded} shows some of the examples.
Full dataset preprocessing and training details are provided in 
Appendix~\ref{app:exp-cifar}.

\begin{figure*}[h]
    \centering
    \begin{subfigure}{0.45\textwidth}
        \centering
        \includegraphics[width=\linewidth]{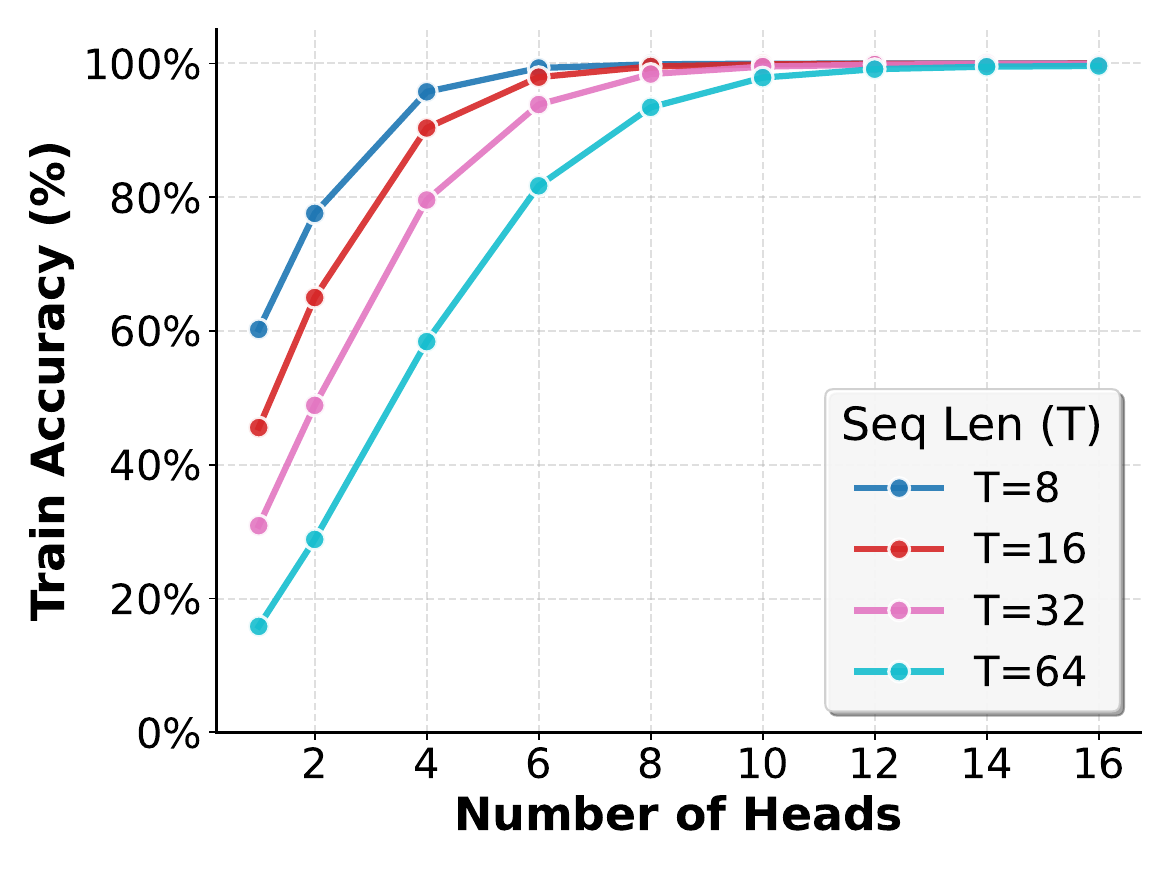}
        \caption{\textbf{Accuracy vs.\ Number of Heads for different $T$ (Text Retrieval).}}
        \label{fig:real-exp-a}
    \end{subfigure}
    \hfill
    \begin{subfigure}{0.45\textwidth}
        \centering
        \includegraphics[width=\linewidth]{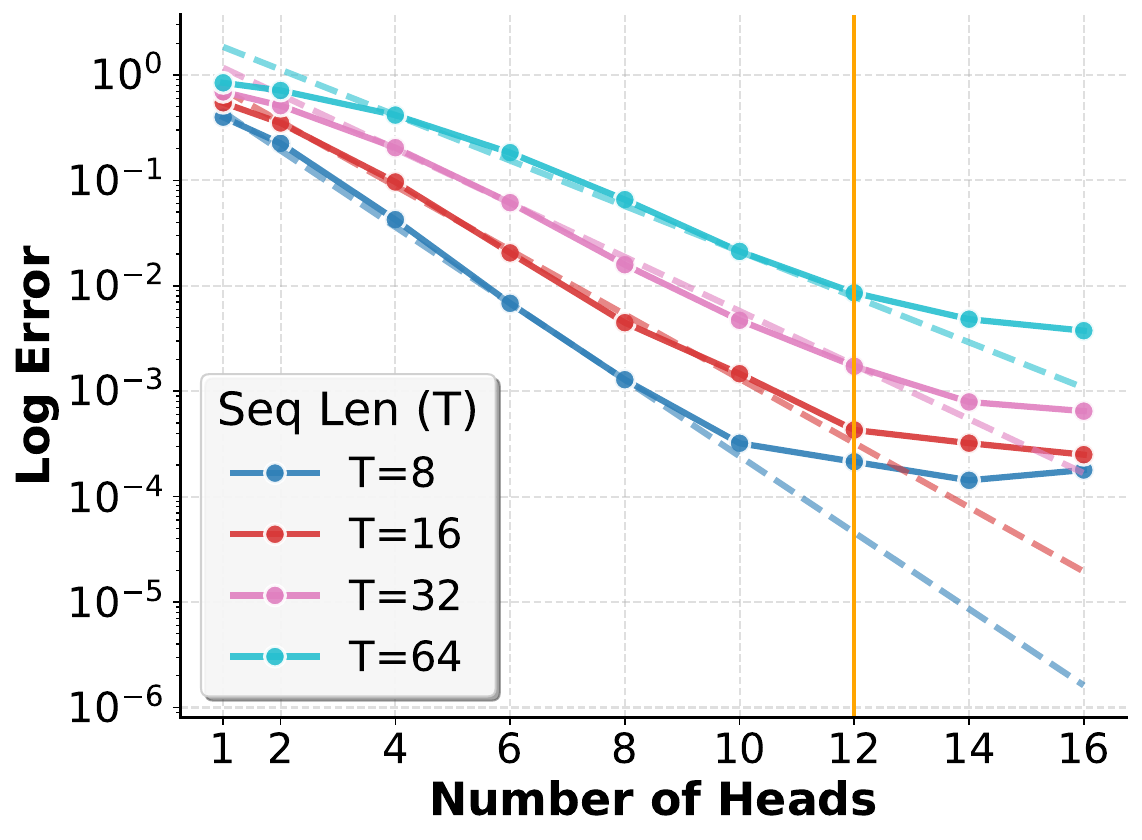}
        \caption{\textbf{Log(1-Accuracy) and its prediction (Text Retrieval)}}
        \label{fig:real-exp-b}
    \end{subfigure}


    \begin{subfigure}{0.33\textwidth}
        \centering
        \includegraphics[width=\linewidth]{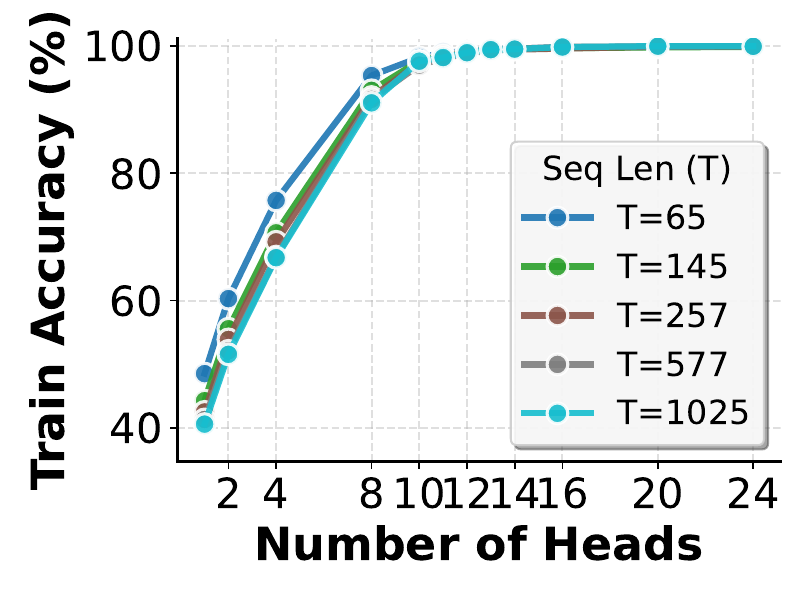}
        \caption{\textbf{Train accuracy ($\%$) vs.\ Number of Heads (Image).}}
        \label{fig:real-exp-c}
    \end{subfigure}
    \hfill
    \begin{subfigure}{0.32\textwidth}
        \centering
        \includegraphics[width=\linewidth]{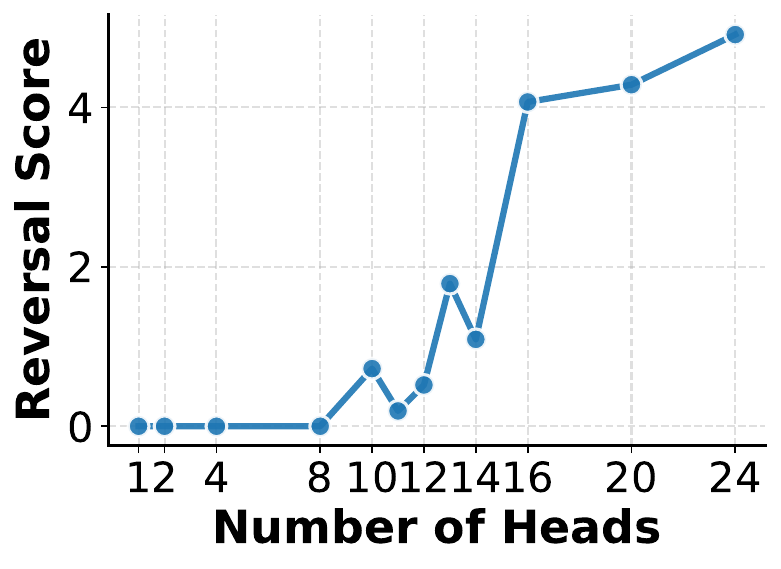}
        \caption{\textbf{Weighted Reversal score vs.\ Head Number (Image) }}
        \label{fig:real-exp-d}
    \end{subfigure}
    \hfill
    \begin{subfigure}{0.32\textwidth}
        \centering
        \includegraphics[width=\linewidth]{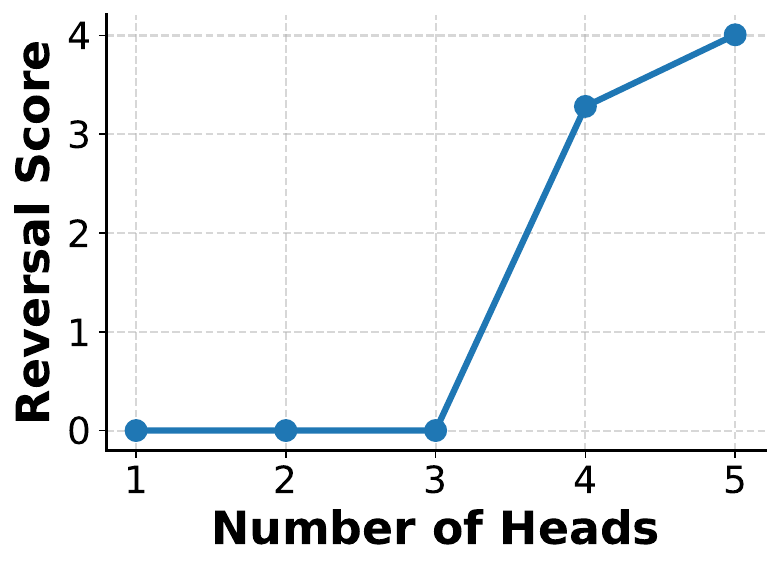}
        \caption{\textbf{Weighted Reversal score vs.\ Head Number (Synthetic) }}
        \label{fig:real-exp-e}
    \end{subfigure}

    \caption{\textbf{Experiments on real datasets.} 
    Training performance with different numbers of heads $h$ across different sequence lengths $T$. 
    (a)~Accuracy vs.\ number of heads for different $T$ in text retrieval; phase transition near $h=12$. Mean and standard deviation see Table~\ref{table:MS_Mean_var}. MRR shows a similar trend, see Fig.~\ref{fig:ms mrr} in the appendix.
    (b)~Phase transition for text retrieval.
    (c)~Accuracy vs.\ number of heads for different $T$ in image classification; phase transition near $h=10$. Mean and standard deviation see Table~\ref{table: image}.
    (d)~Weighted Reversal Score for Image Classification, $err = 1-Accuracy$. The plot becomes positive when $h\geq 10$, indicating phase transition. 
    (e)~Weighted Reversal Score for Synthetic Experiment, it becomes positive at $h=4$, exactly the intrinsic dimension of the task.}
    \label{fig:real-exp}
\end{figure*}

\paragraph{Result analysis.} \label{real-data-analysis}
Both experiments exhibit the same qualitative trend as in the synthetic setting. \\
Figure~\ref{fig:real-exp-a} shows that in the text retrieval experiment, when $h < 12$, accuracy declines as the sequence 
length $T$ increases, consistent with Theorem~\ref{thm2:p2}. Once $h > 12$, this dependence on $T$ disappears, and performance remains stable. Taking $Err(h,T) = 1-\text{Accuracy}(h,T)$ as error, by using $cT^{\beta}\exp(\alpha h/T^\delta)$ to approximate $(Err(h,T))$ in log scale under MAE and drop-outs ($h=1,12,14, 16$ are dropped out as outliers, $\delta = 0.25>0, \alpha = -1.40<0$), figure~\ref{fig:real-exp-b} illustrates that when $h <12$, $-\log(Err(h,T))\propto h/T^\delta$, highly consistent with the order in Theorem~\ref{thm2:p2} under fixed parameter count $M$. The flattening of curves after $h > 12$ is also consistent with theory. 

Figure~\ref{fig:real-exp-c} shows similar trend in image classification, with intrinsic dimension at $h=10$. Figures~\ref{fig:real-exp-d} and \ref{fig:real-exp-e} illustrates weighted reversal score, calculated by $R(h) =\frac{1}{w_h} \sum_{T_1<T_2}\max((err(T_1)-err(T_2)),0)$ with normalization factor $w_h = \max_{T}err(T) - \min_{T}err(T)$, detects the existence of longer $T$ yielding smaller error for this head number $h$. Such phenomenon leads to positive $R(h)$, and it also indicates phase transition as explained in remark~\ref{observation:reverse}. Figure~\ref{fig:real-exp-e} further verified this.

\section{Conclusion}

In this work we investigated the approximation properties of single-layer 
transformers. We first introduced a structured target family, the generalized $D$-retrieval task, that is broad 
enough to capture general sequence-to-vector mappings (Theorem~\ref{thm:thm1}). Within this setting, we analyzed how the 
approximation efficiency of transformers depends on architectural choices, especially the number of head. Our 
results indicate that having a sufficient number of heads leads to efficient approximation, while an 
insufficient number of heads forces the parameter count to grow exponentially with sequence length $T$. We also examined the single-head case, where large 
embedding dimension allows sequence memorization but shifts the complexity to 
the feed-forward block (Theorem~\ref{thm:thm2}). These findings clarify the roles played by 
head count in transformer expressivity.\\

Our experiments on both synthetic and real datasets reveal a non-trivial phase transition around the intrinsic dimension $D$, consistent with theoretical analysis. When the number of heads is below $D$, models exhibit higher error for the same parameter count as sequence length $T$ increases. Once the head count reaches or exceeds $D$, approximation rate becomes independent of sequence lengths $T$. This transition is also observed in real-world datasets with deeper architectures, indicating that the notion of intrinsic dimension is not only theoretical but also practically relevant. In particular, beyond fully training models, analyzing head contributions early in training to estimate how many heads meaningfully affect the output, or training multiple models with varying depths and head counts while tracking how error scales with $T$ are potential ways to probe the task's intrinsic dimension. These experiments might demonstrate whether the inferred intrinsic dimension is stable across architectures, thereby informing head-count selection and the head number to retain under pruning.

\paragraph{Limitations.}
We conclude by noting several limitations of this study. Firstly, although the analyzed target class is dense, the phenomena of interest are most naturally manifested in retrieval-style tasks aligned with our setting. Secondly, our analysis is restricted to single-layer transformers; while experiments on real datasets supports Conjecture~\ref{Conjecture1} in deeper architectures, a rigorous multi-layer theory remains open. Finally, the tradeoff between sequence memorization and pattern learning—observed for shorter sequences (cf.\ Remark~\ref{observation:reverse})—has not yet been established rigorously and warrants further investigation.

\newpage
\subsection*{Acknowledgments}
This research is supported by the National Research Foundation, Singapore, under the NRF fellowship (project No. NRF-NRFF13-2021-0005). The computational work for this article was fully performed on resources of the National Supercomputing Centre, Singapore (https://www.nscc.sg).
\bibliography{iclr2026_conference}
\bibliographystyle{iclr2026_conference}
\newpage
\appendix

\section{Proofs of Main Theorems}
\label{app:proofs}

\subsection{Proof of Theorem~\ref{thm:thm1}}
\label{app:proof-thm1}
\paragraph{Proof Sketch.}  
The proof proceeds in three steps.  
First, by Lemma~\ref{lem:closure}, we approximate a broader function class that relaxes 
the smoothness requirements and the assumptions in Assumption~\ref{assump:target}.  
Second, Lemmas~\ref{lem:os} and~\ref{lem:selector} show that by constructing appropriate $S_i$, we can faithfully recover all information from the original input sequence with simple $f_i$.  
Finally, the outer function 
$F_0$ can be applied to approximate an arbitrary sequence-to-vector target within this class.  
Together, these steps establish the result.

\begin{proof}[Proof of Theorem~\ref{thm:thm1}]
To prove Theorem~\ref{thm:thm1}, we first establish a few auxiliary lemmas.
\begin{lemma}[Relaxed target class and closure equivalence]\label{lem:closure}
Let $\widetilde{\mathcal F}_D^{d,T}$ be defined as in ~\ref{eq:target} but with only $f_i\in C([0,1]^d,[0,1])$, i.e., we drop ''unique minimizer'', ''pairwise distinct''and ''PD Hessian''.Set $\mathcal F^{d,T}:=\bigcup_{D\ge1}\mathcal F_D^{d,T}$ and $\widetilde{\mathcal F}^{d,T}:=\bigcup_{D\ge1}\widetilde{\mathcal F}_D^{d,T}$. Then
\[
\overline{\mathcal F^{d,T}}^{\,\|\cdot\|_\infty}=\overline{\widetilde{\mathcal F}^{d,T}}^{\,\|\cdot\|_\infty}\quad\text{on } \mathcal X_T.
\]
\end{lemma}
\begin{proof}
Fix $\widetilde H(X_T)=\widetilde F_0\!\big(\widetilde z_1(X_T),\dots,\widetilde z_D(X_T)\big)\in\widetilde{\mathcal F}_D^{d,T}$ with
$\widetilde z_i(X_T)=\min_{t\in S_i}\widetilde f_i(x(t))$ and $\widetilde f_i\in C([0,1]^d,[0,1])$.
Let $\varepsilon>0$.
We will construct $H\in\mathcal F^{d,T}$ with $\|H-\widetilde H\|_\infty\le \varepsilon$.
Firstly, by Stone--Weierstrass, choose $p_i\in C^\infty([0,1]^d)$ so that
$\|p_i-\widetilde f_i\|_\infty\le \eta$, where $\eta>0$ will be fixed later.
Because the uniform approximation can slightly leave $[0,1]$, compose with a smooth strictly increasing squashing $s:[-c,1+c]\to[0,1]$ with $s(u)=u$ on $[0,1]$ and $\|s\circ p_i - p_i\|_\infty\le \eta$ (for small enough $c>0$), and replace $p_i$ by $s\circ p_i$. We still write $p_i$ and retain $\|p_i-\widetilde f_i\|_\infty\le 2\eta$.

Secondly, let $\xi_i\in\arg\min_{x\in[0,1]^d} p_i(x)$ (nonempty by compactness).
Pick $r\in(0,\tfrac{1}{4})$ small and a $C^\infty$ bump $\phi_i$ supported in $B(\xi_i,2r)\cap[0,1]^d$, with $\phi_i(\xi_i)=1$, $\nabla\phi_i(\xi_i)=0$, and with
$\nabla^2 \phi_i(\xi_i)$ \emph{negative definite}.\footnote{For instance take $\phi_i(x)=\psi(\|x-\xi_i\|^2/r^2)$ with $\psi(0)=1$, $\psi'<0$ near $0$, $\psi\equiv0$ on $[1,\infty)$; then $\nabla^2\phi_i(\xi_i)\prec 0$ and its norm scales like $r^{-2}$.}
Define, for parameters $\delta_{i,1},\delta_{i,2}>0$ to be fixed,
\[
g_i(x)\ :=\ p_i(x)\;-\;\delta_{i,1}\,\phi_i(x)\;+\;\delta_{i,2}\,\phi_i(x)\,\|x-\xi_i\|^2 .
\]
(i) Since $g_i(\xi_i)=p_i(\xi_i)-\delta_{i,1}$ while $g_i(x)\ge p_i(x)$ whenever $\phi_i(x)=0$ and $g_i(x)>p_i(x)$ for $x\in B(\xi_i,2r)\setminus\{\xi_i\}$, we get that \emph{$\xi_i$ is the unique global minimizer} of $g_i$.

(ii) At $\xi_i$, because $\nabla\phi_i(\xi_i)=0$,
\[
\nabla^2 g_i(\xi_i)
= \nabla^2 p_i(\xi_i) \;-\;\delta_{i,1}\,\nabla^2\phi_i(\xi_i) \;+\; 2\delta_{i,2} I .
\]
Here $-\nabla^2\phi_i(\xi_i)\succ 0$, so choosing $(\delta_{i,1},\delta_{i,2})$ suitably makes $\nabla^2 g_i(\xi_i)\succ0$ (PD).
Because $\|\phi_i\|_\infty\le 1$ and $\|\,\|x-\xi_i\|^2\|_\infty\le d$ on $[0,1]^d$,
\[
\|g_i-p_i\|_\infty \;\le\; \delta_{i,1}\;+\;\delta_{i,2}\,(2r)^2 .
\]
Hence, by taking $r$ small and then $\delta_{i,1},\delta_{i,2}$ small (using the $r^{-2}$ scaling in $\nabla^2\phi_i(\xi_i)$ to keep the Hessian PD), we can ensure both PD at $\xi_i$ and $\|g_i-p_i\|_\infty\le \eta$.

Thirdly, we use tiny translation to remove distinctiveness. It may happen that $\xi_i=\xi_{i'}$ for some $i\neq i'$.
Choose pairwise distinct small vectors $v_i\in\mathbb{R}^d$ and fix a smooth cutoff
$\chi\in C^\infty([0,1]^d,[0,1])$ that equals $1$ on $[r,1-r]^d$ and vanishes near the boundary.
Define a $C^\infty$ diffeomorphism of the cube,
\[
\Phi_i(x)\ :=\ x\;-\;\varepsilon_i\,\chi(x)\,v_i ,
\qquad \text{with } \varepsilon_i>0 \text{ small}.
\]
Then $\Phi_i$ is arbitrarily close to the identity in $C^1$ for small $\varepsilon_i$, maps $[0,1]^d$ to itself, and
$h_i:=g_i\circ \Phi_i$ has a (unique) minimizer at $x^{(i)}:=\Phi_i^{-1}(\xi_i)$.
For different $i$, these points are distinct if the $v_i$'s are distinct and $\varepsilon_i$'s are small but nonzero.
Moreover, because $\nabla g_i(\xi_i)=0$, the Hessian at $x^{(i)}$ satisfies
\[
\nabla^2 h_i(x^{(i)}) \;=\; D\Phi_i(x^{(i)})^{\!\top}\,\nabla^2 g_i(\xi_i)\,D\Phi_i(x^{(i)}) \;\succ\; 0,
\]
so PD is preserved.
Since $g_i$ is Lipschitz on the compact cube, $\|h_i-g_i\|_\infty\le L_i\,\varepsilon_i$ for some $L_i$, hence by taking $\varepsilon_i$ small we get $\|h_i-g_i\|_\infty\le \eta$.

Finally, if needed, compose with the same strictly increasing squashing $s$ as in Step~1 and set
\[
f_i\ :=\ s\circ h_i \in C^2([0,1]^d,[0,1]).
\]
Because $s$ is strictly increasing, it preserves the minimizer location and, at the minimizer $x^{(i)}$, 
$\nabla^2 (s\circ h_i)(x^{(i)}) = s'\!\big(h_i(x^{(i)})\big)\,\nabla^2 h_i(x^{(i)})\succ0$.
Also $\|f_i-h_i\|_\infty\le \eta$ by construction.

Collecting the bounds from previous deduction:
\[
\|f_i-\widetilde f_i\|_\infty
\;\le\; \underbrace{\|p_i-\widetilde f_i\|_\infty}_{\le 2\eta}
+\underbrace{\|g_i-p_i\|_\infty}_{\le \eta}
+\underbrace{\|h_i-g_i\|_\infty}_{\le \eta}
+\underbrace{\|f_i-h_i\|_\infty}_{\le \eta}
\;\le\; 5\eta .
\]
For each $i$, the map $u\mapsto \min_{t\in S_i} u_t$ is $1$-Lipschitz in $\|\cdot\|_\infty$.
Hence the corresponding features
$\bar{z}_i(X_T):=\min_{t\in S_i} f_i(x(t))$ and $\widetilde z_i(X_T):=\min_{t\in S_i}\widetilde f_i(x(t))$
satisfy $\|\bar{z}_i-\widetilde z_i\|_\infty \le 5\eta$.
Let $\omega_{\widetilde F_0}$ be a modulus of continuity of $\widetilde F_0$ on $[0,1]^D$.
Choose $\eta$ so small that $\omega_{\widetilde F_0}(5\eta)\le \varepsilon/2$.
Then
\[
\big\|\widetilde F_0\big(\bar{z}(X_T)\big)-\widetilde F_0\big(\widetilde z(X_T)\big)\big\|_\infty
\le \varepsilon/2 .
\]
Finally, approximate $\widetilde F_0$ uniformly on $[0,1]^D$ by some $F_0\in C^1([0,1]^D)$ within $\varepsilon/2$ (Stone--Weierstrass). Setting
\[
H(X_T)\ :=\ F_0\big(\bar{z}_1(X_T),\ldots,\bar{z}_D(X_T)\big)
\in \mathcal F^{d,T} ,
\]
we obtain
\[
\|H-\widetilde H\|_\infty
\ \le\
\underbrace{\|F_0(\bar{z})-\widetilde F_0(\bar{z})\|_\infty}_{\le \varepsilon/2}
+\underbrace{\|\widetilde F_0(z)-\widetilde F_0(\widetilde z)\|_\infty}_{\le \varepsilon/2}
\ \le\ \varepsilon .
\]
This shows $\widetilde{\mathcal F}^{d,T}\subset \overline{\mathcal F^{d,T}}^{\,\|\cdot\|_\infty}$.
The reverse inclusion is simple, hence we have the lemma.

\end{proof}

\begin{remark}
Thus, we now focus on the relaxed class $\widetilde{\mathcal F}^{d,T}$ and Lemma \ref{lem:closure} lifts the result to the original class ${\mathcal F}^{d,T}$.
\end{remark}

\begin{lemma}[Order-statistic in the relaxed class]\label{lem:os}
Without loss of generation, suppose $4|T$. Let $m=\frac{T}{4}$.  
For each $j\in[d]$ and $X_T=\{x(1),\dots,x(T)\}$, define
\[
U_j(X_T):=\max_{\substack{B\subseteq[T]\\|B|=m}}\ \min_{u\in B} x(u)_j,
\]
and for each fixed $t\in[T]$,
\[
Y_{t,j}(X_T):=\max_{\substack{A\subseteq[T]\\|A|=m,\ t\in A}}\ \min_{u\in A} x(u)_j,\qquad
Z_{t,j}(X_T):=\max_{\substack{A\subseteq[T]\\|A|=m,\ t\in A}}\ \min_{u\in A} (1-x(u)_j).
\]
Let $v_{1,j}\ge\cdots\ge v_{T,j}$ be the sorted values of $\{x(1)_j,\dots,x(T)_j\}$ and set $U_j=v_{m,j}$.
For the multi-set $\{x(u)_j:u \in [T]\}$, let $v_{1,j}\ge\cdots\ge v_{T,j}$ (nonincreasing) and $w_{1,j}\le\cdots\le w_{T,j}$ (nondecreasing). Then we have
\[
Y_{t,j}=\min\{x(t)_j,\ v_{m,j}\},\qquad
1-Z_{t,j}=\max\{x(t)_j,\ w_{m,j}\}.
\]
In particular:
\[
x(t)_j\le v_{m,j}\ \Rightarrow\ Y_{t,j}=x(t)_j,\qquad
x(t)_j\ge w_{m,j}\ \Rightarrow\ 1-Z_{t,j}=x(t)_j.
\]
\end{lemma}

\begin{proof}
Among all $m$-subsets $B$, the maximum of $\min_{u\in B}x(u)_j$ is attained by picking the $m$ largest coordinates, so $U_j=v_{m,j}$. Forcing $t\in A$, the choice of the other $m-1$ indices to maximize the minimum is the $m-1$ largest among $\{x(u)_j:u\neq t\}$,
hence $Y_{t,j}=\min\{x(t)_j,v_{m,j}\}$. For $Z_{t,j}$, note that $\min_{u\in A}(1-x(u)_j)=1-\max_{u\in A}x(u)_j$, so maximizing it over $A$ is the same as minimizing $\max_{u\in A}x(u)_j$, which picks $t$ plus the $(m\!-\!1)$ smallest leading to $1-Z_{t,j}$. The particular statements follow immediately.
\end{proof}

\begin{lemma}[Smooth selector]\label{lem:selector}
For $q>0$ define
\[
\widehat x_q(t)_j\ :=\
\frac{e^{q (\,1-Z_{t,j}-w_{m,j}\,)^2}\,(1-Z_{t,j})
      +e^{q (\,Y_{t,j}-v_{m,j}\,)^2}\,Y_{t,j}}
     {e^{q (\,1-Z_{t,j}-w_{m,j}\,)^2}
      +e^{q (\,Y_{t,j}-v_{m,j}\,)^2}}.
\]
Then $\widehat x_q(t)_j\to x(t)_j$ uniformly on $\mathcal X_T$ as $q\to\infty$.
\end{lemma}

\begin{proof}
From Lemma~\ref{lem:os}, if $x(t)_j\ge U_j$, we have $Y_{t,j}=v_{m,j}$ and $(Y_{t,j}-v_{m,j})^2=0$ while $1-Z_{t,j}=x(t)_j$ and $(1-Z_{t,j}-w_{m,j})^2>0$. Thus, as $q\to\infty$, the weight concentrates on $(1-Z_{t,j})=x(t)_j$. If $x(t)_j\le U_j$ and $w_{m,j}\ge x(t)_j$, we have $(1-Z_{t,j}-w_{m,j})^2=0$ while $Y_{t,j}=x(t)_j$ and $(Y_{t,j}-v_{m,j})^2>0$ concentrating on $Y_{t,j}=x(t)_j$. If $w_{m,j}\le x(t)_j\le U_j$, we have $1-Z_{t,j}=Y_{t,j}=x(t)_j$, so either of three settings leads to $x(t)_j$. The compactness of $\mathcal X_T$ gives us the uniform property.
\end{proof}

Now, we begin our formal proof for Theorem~\ref{thm:thm1}.

By Lemma~\ref{lem:closure}, for any $F\in C(\mathcal X_T)$ and $\varepsilon>0$, it suffices to construct an $\widetilde H\in\widetilde{\mathcal F}^{d,T}$ with $\|F-\widetilde H\|_\infty\le \varepsilon/2$, since the lemma~\ref{lem:closure} could lift it to $\mathcal F^{d,T}$ with another $\varepsilon/2$.

Fix $m=\frac{T}{4}$. For each coordinate $j$ and each $S\subseteq[T]$ with $|S|=m$, include the relaxed primitives
\[
z^{(j,S)}(X_T):=\min_{t\in S} x(t)_j,\qquad
\bar z^{(j,S)}(X_T):=\min_{t\in S} (1-x(t)_j).
\]
To form the thresholds $v_{m,j}$ and $w_{m,j}$ needed in Lemma~\ref{lem:os}, additionally include:
\[
\min_{t\in S} x(t)_j\quad\text{for all }S\subseteq[T]\setminus\{t\}\text{ with }|S|=m,
\]
and
\[
\min_{t\in S} (1-x(t)_j)\quad\text{for all }S\subseteq[T]\setminus\{t\}\text{ with }|S|=T-m,
\]
Using smooth log-sum-exp (softmax) in the outer function $\widetilde F_0$, we can recover the subset-wise maximum required to compute $U_j$, $Y_{t,j}$, $Z_{t,j}$, $v_{m,j}$ and $w_{m,j}$ from these primitives.

By Lemma~\ref{lem:selector}, for any $\delta>0$ there exists $q$ such that
\[
\max_{X_T\in\mathcal X_T}\max_{t\in[T],\,j\in[d]}\big|\widehat x_q(t)_j - x(t)_j\big|\ \le\ \delta.
\]
By uniform continuity of $F$ on the compact $\mathcal X_T$, choose $\delta$ so that this implies $|F(X_T)-F(\widehat X_q(T))|\le \varepsilon/4$ for all $X_T$, where $\widehat X_q(T)$ stacks the coordinates $\widehat x_q(t)_j$. We approximate the continuous map $u\mapsto F(u)$ on $[0,1]^{dT}$ uniformly by a polynomial $P$ within $\varepsilon/4$ (Stone–Weierstrass). Define
\[
\widetilde H(X_T):=\big(P\circ \mathrm{vec}\big)(\widehat X_q(T)),
\]
which is a $C^1$ function of the inner features. Hence $\widetilde H\in\widetilde{\mathcal F}^{d,T}$ and
\[
\|F-\widetilde H\|_\infty
\le \underbrace{\|F-F\circ \widehat X_q\|_\infty}_{\le \varepsilon/4}
+ \underbrace{\|F\circ \widehat X_q - P\circ \widehat X_q\|_\infty}_{\le \varepsilon/4}
\le \varepsilon/2.
\]

Apply Lemma \ref{lem:closure} to replace each relaxed primitives by admissible $C^2$ functions with unique minimizers
and to replace $\widetilde F_0$ by function $F_0$ so that the final error increases by at most $\varepsilon/2$.
This leads to $f\in\mathcal F_D^{d,T}$ with $\|F-f\|_\infty\le \varepsilon$.

\end{proof}

\subsection{Proof of Theorem~\ref{thm:thm2}}
\label{app:proof-thm2}

\subsubsection{Proof of Theorem~\ref{thm2:p1} }
Here we prove Theorem \ref{thm2:p1}
\paragraph{Proof Sketch.}  
The idea is straightforward. Each attention head is assigned to approximate one term 
$\min_{t \in S_i} f_i(x(t))$. Once these components are extracted, the outer function 
$F_0$ can be approximated by a suitable $\hat{F}$, completing the construction.

\begin{proof}[Proof of Theorem~\ref{thm2:p1}: Sufficient expressivity with $D$ heads]

Fix $d,T$ and $\alpha\in(0,1)$, and let $\varepsilon>0$ be given. 
Throughout the proof, constants depending only on $(d,D,T)$ are absorbed into 
$C_{d,D,T}>0$ which may change from line to line.

In the display~\eqref{eq:transformer}, each head produces an $n$-dimensional vector and $\operatorname{Concat}_{i=1}^h$ gives a vector in $\R^{nh}$ before $\hat F$. For the construction, we realize the usual \emph{block-by-head} parameterization, which means that
the encoder outputs a block-decomposed embedding
\[
\hat x(t)\;=\;\bigl(\hat x^{(1)}(t),\ldots,\hat x^{(D)}(t)\bigr)\in\R^{2D},
\qquad
\hat x^{(i)}(t)\in\R^2,
\]
and the $i$-th head only reads the $i$-th block via block-diagonal
$W_{Q,i},W_{V,i}$ (entries bounded by~$1$).  
This keeps the parameter counts within the same order.  We therefore set the \emph{per-head
embedding dimension} to $n=2$.

Firstly, from Assumption (A2), for any $\delta>0$ there exist two-layer FFNs
$\Psi_{i,\delta}:[0,1]^d\to[0,1]$ such that
\begin{equation}\label{eq:f-approx}
\max_{x\in[0,1]^d}\,|f_i(x)-\Psi_{i,\delta}(x)|\le\delta,
\qquad
\mathrm{width}(\Psi_{i,\delta})\le \frac{C_2}{\delta^{\gamma_f}},
\end{equation}
where $\gamma_f>0$ is the exponent from (A2).
Define for each head $i$, the position gate
\[
r_i(s):=\begin{cases}
0,& s\in S_i,\\[2pt]
-1,& s\notin S_i,
\end{cases}
\qquad s\in[T].
\]
(Recall $S_i\subset[T]$ with $|S_i|\ge \alpha T$ by~\eqref{eq:zi}.)
We implement the encoder $P_\phi$ so that its $i$-th block is
\begin{equation}\label{eq:encoder-block}
\hat x^{(i)}(t)\;=\;\bigl(\,\Psi_{i,\delta}(x(t))\,,\, r_i(t)\,\bigr)\in[0,1]\times\{-1,0\}
\subset[-1,1]^2 .
\end{equation}
This choice follows $\|\hat x(t)\|_2\le \sqrt{2D}$. After a fixed rescaling (absorbed into $\beta$), this meets the norm constraint.

Secondly, we would like to use head-wise attention to isolate the minimum on $S_i$. 
For each head $i$, we take a single attention logit ($m_h=1$) by choosing
\[
W_{O,i}=I,\qquad
W_{K,i}=\begin{bmatrix}-1&\;1\end{bmatrix},\qquad
W_{Q,i}\hat c_0=1,\qquad
W_{V,i}=I_2 .
\]
All entries are within the allowed bound~$1$. With the block~\eqref{eq:encoder-block}, the (pre-softmax) score of token~$t$ in head~$i$ is
\begin{equation}\label{eq:rho-def}
\rho_i(t)\;=\;(W_{K,i}\hat x^{(i)}(t))^\top (W_{Q,i}\hat c_0)
\;=\;-\Psi_{i,\delta}\bigl(x(t)\bigr)+r_i(t).
\end{equation}
Let $\sigma[\rho_i]$ be the softmax~\eqref{eq:softmax} with $\beta>0$.
Define the head-$i$ value readout (first coordinate of the head output)
\begin{equation}\label{eq:ztilde-def}
\tilde z_i(X_T)\;:=\;\sum_{t=1}^T \sigma[\rho_i](t)\;\Psi_{i,\delta}\bigl(x(t)\bigr).
\end{equation}
Here, the second coordinate is unused. $\hat F$ could ignore it via a fixed linear projection,
counted in the constant $C_{d,D,T}$.

Now, we give a uniform bound on $S_{i}$. Take $a_t:=\Psi_{i,\delta}(x(t))\in[0,1]$ and split the sum into $S_i$ and $S_i^c$. From $r_i(t)=0$ on $S_i$ and $r_i(t)=-1$ on $S_i^c$, we have 
\[
\sigma[\rho_i](t)
=\frac{e^{\beta(-a_t+r_i(t))}}{\sum_{u\in S_i} e^{-\beta a_u}+\sum_{u\in S_i^c} e^{-\beta (a_u+1)}}
\le
\begin{cases}
\displaystyle \frac{e^{-\beta a_t}}{\sum_{u\in S_i} e^{-\beta a_u}},& t\in S_i,\\[10pt]
e^{-\beta}\cdot \displaystyle \frac{e^{-\beta a_t}}{\sum_{u\in S_i} e^{-\beta a_u}},& t\in S_i^c.
\end{cases}
\]
Hence, we have 
\begin{equation}\label{eq:split}
\tilde z_i
=\sum_{t\in S_i}\sigma[\rho_i](t)a_t+\sum_{t\in S_i^c}\sigma[\rho_i](t)a_t
\le \underbrace{\frac{\sum_{t\in S_i} a_t e^{-\beta a_t}}{\sum_{u\in S_i} e^{-\beta a_u}}}_{\text{Gibbs mean on }S_i}
\;+\; e^{-\beta}\,\frac{\sum_{t\in S_i^c} a_t e^{-\beta a_t}}{\sum_{u\in S_i} e^{-\beta a_u}} .
\end{equation}
To simplify, we denote $a_*:=\min_{t\in S_i} a_t$ and $b_t:=a_t-a_*\in[0,1]$ for $t\in S_i$. Then, we have
\[
\frac{\sum_{t\in S_i} a_t e^{-\beta a_t}}{\sum_{u\in S_i} e^{-\beta a_u}}
=a_*+\frac{\sum_{t\in S_i} b_t e^{-\beta b_t}}{\sum_{u\in S_i} e^{-\beta b_u}}
\;\le\; a_*+\sum_{t\in S_i} b_t e^{-\beta b_t}
\;\le\; a_*+\frac{|S_i|-1}{e\beta},
\] 
The inequality comes from $\sup_{b\in[0,1]} b e^{-\beta b}=e^{-1}/\beta$ for $\beta\ge 1$ and one of the
$b_t$ is~$0$, so the denominator in the middle fraction is $\ge 1$. For the $S_i^c$ term in~\ref{eq:split}, 
we use $a_t\le 1$ and
$\sum_{u\in S_i} e^{-\beta a_u}\ge 1$ to get
\[
e^{-\beta}\,\frac{\sum_{t\in S_i^c} a_t e^{-\beta a_t}}{\sum_{u\in S_i} e^{-\beta a_u}}
\;\le\; e^{-\beta}\,\bigl|S_i^c\bigr|
\;\le\; e^{-\beta}\,T .
\]
Combining the two bounds, we have the uniform estimate:
\begin{equation}\label{eq:softmin-bound}
\min_{t\in S_i}\Psi_{i,\delta}\bigl(x(t)\bigr)
\;\le\;\tilde z_i(X_T)
\;\le\;\min_{t\in S_i}\Psi_{i,\delta}\bigl(x(t)\bigr)
\;+\;\frac{|S_i|-1}{e\beta}+T e^{-\beta}
\qquad(\beta\ge 1).
\end{equation}
In particular, since $|S_i|\le T$, there is a constant $C_T$ with
\begin{equation}\label{eq:softmin-simpler}
0\;\le\;\tilde z_i(X_T)-\min_{t\in S_i}\Psi_{i,\delta}\bigl(x(t)\bigr)
\;\le\; C_T\Bigl(\frac{1}{\beta}+e^{-\beta}\Bigr),\qquad C_T:=\max\{T/e,\;T\}.
\end{equation}

Thirdly, we need to lift bounds from $\tilde z_i$ to $z_i$. From ~\eqref{eq:f-approx} and the definition of $z_i$,
\[
\Bigl|\min_{t\in S_i} f_i\bigl(x(t)\bigr)-\min_{t\in S_i}\Psi_{i,\delta}\bigl(x(t)\bigr)\Bigr|
\;\le\;\delta .
\]
Together with~\eqref{eq:softmin-simpler},
\begin{equation}\label{eq:zi-tildezi}
\bigl|\tilde z_i(X_T)-\bar{z}_i(X_T)\bigr|
\;\le\;\delta + C_T\Bigl(\frac{1}{\beta}+e^{-\beta}\Bigr)
\qquad\text{for all }X_T\in\mathcal{X}_T,\ i=1,\dots,D.
\end{equation}
Let $L_0:=\sup_{z\in[0,1]^D}\|\nabla F_0(z)\|_1<\infty$ (compactness and $C^1$).
Choose
\[
\delta:=\frac{\varepsilon}{4L_0D},\qquad
\beta\ge \beta_\varepsilon:=\max\!\left\{1,\ \frac{4C_T L_0 D}{\varepsilon},\ \log\!\Bigl(\frac{4C_T L_0 D}{\varepsilon}\Bigr)\right\}.
\]
Then by~\eqref{eq:zi-tildezi},
\begin{equation}\label{eq:coord-error}
\|\tilde z(X_T)-z(X_T)\|_\infty\;\le\;\frac{\varepsilon}{2L_0D}
\qquad\text{for all }X_T\in\mathcal{X}_T.
\end{equation}

Finally, we constrcut the approximation for $F_{0}$ and count the number of parameter. By Assumption~(A1), there exists a two-layer FFN $\Phi_{\delta_0}:\,[0,1]^D\to\R$
with width $\le C_1/\delta_0^{\gamma_0}$ (for some $\gamma_0>0$) such that
\[
\max_{z\in[0,1]^D}\,|F_0(z)-\Phi_{\delta_0}(z)|\le \delta_0 .
\]
Set $\delta_0:=\varepsilon/2$.
Define the model's final feed-forward $\hat F$ to project
$\R^{2D}\to\R^D$ by keeping the first coordinate of each head (a fixed linear
map with entries in $\{0,1\}$) and apply $\Phi_{\delta_0}$.  

Then for all $X_T$, we have
\begin{align*}
\bigl|\hat F(\operatorname{Concat}_i(\cdot)) - F_0\bigl(z(X_T)\bigr)\bigr|
&\le \bigl|\Phi_{\delta_0}\bigl(\tilde z(X_T)\bigr)-\Phi_{\delta_0}\bigl(z(X_T)\bigr)\bigr|
     +\bigl|\Phi_{\delta_0}\bigl(z(X_T)\bigr)-F_0\bigl(z(X_T)\bigr)\bigr|\\
&\le L_0\,\|\tilde z(X_T)-z(X_T)\|_1+\delta_0\\
&\le L_0 D\,\|\tilde z(X_T)-z(X_T)\|_\infty + \varepsilon/2\\
&\le \varepsilon/2+\varepsilon/2\;=\;\varepsilon,
\end{align*}
where we used~\eqref{eq:coord-error} in the last inequality.

Here, the trainable components are composed of three parts:
\begin{itemize}
    \item the $D$ subnetworks $\Psi_{i,\delta}$ inside the encoder blocks~\eqref{eq:encoder-block};
    \item the fixed-size projections $W_{Q,i},W_{K,i},W_{V,i}$ (size~$O(D)$ and independent of $\varepsilon$); 
    \item the two-layer FFN $\Phi_{\delta_0}$ used inside $\hat F$.
\end{itemize}
Thus
\[
M\;\le\; C'\,D\cdot \frac{1}{\delta^{\gamma_f}} \;+\; C''\cdot \frac{1}{\delta_0^{\gamma_0}}
\;+\;C'''\qquad\text{for constants }C',C'',C'''=C_{d,D,T}.
\]
With $\delta=\Theta(\varepsilon)$ and $\delta_0=\Theta(\varepsilon)$ chosen above,
\[
M\;\le\; \frac{C_{d,D,T}}{\varepsilon^{\gamma}},
\qquad \gamma:=\max\{\gamma_f,\gamma_0\},
\]
and the construction uses $h=D$ heads with per-head dimension $n=2$ and achieves
$\varepsilon$-approximation on $\mathcal{X}_T$.  This proves Theorem~\ref{thm2:p1}.
\end{proof}

\subsubsection{Proof of Theorem~\ref{thm2:p2} }
\paragraph{Proof Sketch.}  
The argument proceeds in two parts. The core idea is to construct two sequences whose 
representations after the attention layer are indistinguishably close, on the order of 
$O(\epsilon^{k+1})$, yet whose target outputs differ by at least $3\epsilon$. 
Lemma~\ref{lem:FFN-rate} then implies the lower bound on the parameter count required for 
approximation.  

Using Lemmas~\ref{lem:Hessian}, \ref{lem:gradient}, and~\ref{lem:separation basin}, 
we obtain $D$ disjoint neighborhoods around the minima $x^{(i)}$. 
Since $D > s = h$, there exists at least one neighborhood not selected by the $s$ heads. 
Within this region, the pigeonhole principle guarantees the existence of two distinct 
subsequences. By carefully designing these subsequences, we ensure that their outputs 
after the attention layer are nearly indistinguishable, while their target values differ 
by at least $3\epsilon$. Extending them to full sequences completes the construction.

We now turn to the full proof. 
To establish Theorem~\ref{thm2:p2}, 
we begin by introducing several auxiliary lemmas that will serve as building blocks for the argument. Lemma~\ref{lem:Hessian}, \ref{lem:gradient}, and~\ref{lem:separation basin} are only to set up the approximation problem into a more tractable form.

\begin{lemma}\label{lem:FFN-rate}
    Let $v_1,v_2 \in \mathbb{R}^n$. Suppose 
\[
    \|v_1 - v_2\|_2 \le A
    \quad\text{and}\quad
    \|\hat{F}(v_1) - \hat{F}(v_2)\| \ge B ,
\]
where $\hat{F}:\mathbb{R}^n \to \mathbb{R}^m$ is a two-layer feed-forward network 
satisfying the constraints stated above. Then $\hat{F}$ must use at least
\[
    \Omega\!\left(\frac{B}{A\sqrt{n}}\right)
\]
parameters.
\end{lemma}

\begin{proof}[Proof of Lemma~\ref{lem:FFN-rate}]
Let $\Delta x := v_1 - v_2$ and $\Delta F := \hat{F}(v_1) - \hat{F}(v_2)$.
Suppose the two-layer network with width $p$ be
\[
\hat{F}(x) \;=\; V\,\sigma(Ux+b)+c,
\]
where $U\in\mathbb{R}^{p\times n}$, $V\in\mathbb{R}^{m\times p}$, $b\in\mathbb{R}^p$, $c\in\mathbb{R}^m$,
$\sigma$ is $1$-Lipschitz acting coordinate-wise and every entry of $U,V,b,c$ has magnitude at most $1$.

For the $j$-th output coordinate, we have
\[
\Delta F_j
= \sum_{r=1}^p V_{jr}\!\left(\sigma(u_r^\top v_1 + b_r)-\sigma(u_r^\top v_2 + b_r)\right),
\]
where $u_r^\top$ is the $r$-th row of $U$.
Using the $1$-Lipschitz property of $\sigma$ and Cauchy--Schwarz inequality, we have
\[
|\Delta F_j|
\le \sum_{r=1}^p |V_{jr}|\,|u_r^\top \Delta x|
\le \sum_{r=1}^p |V_{jr}|\,\|u_r\|_2\,\|\Delta x\|_2.
\]
By the entrywise weight bound, $\|u_r\|_2 \le \sqrt{n}$ and $|V_{jr}|\le 1$. Therefore, for all $j$, we have
\begin{equation}\label{eq:per-coord-bound}
|\Delta F_j| \;\le\; p\,\sqrt{n}\,\|\Delta x\|_2.
\end{equation}
Let $\|\cdot\|$ be the norm used in the lemma statement. By norm equivalence in finite dimensions, there exists $c_m\in(0,1]$ depending only on the chosen norm and $m$ such that
\[
\|y\| \le \frac{1}{c_m}\|y\|_\infty \quad \text{for all } y\in\mathbb{R}^m.
\]
$\|\Delta F\|\ge B$ implies $\|\Delta F\|_\infty \ge c_m B$, so there is some $j^\star$ with
\[
|\Delta F_{j^\star}| \;\ge\; c_m B.
\]
Combining this with \eqref{eq:per-coord-bound} and $\|\Delta x\|_2\le A$, we have 
\[
c_m B \;\le\; p\,\sqrt{n}\,A
\quad\Longrightarrow\quad
p \;\ge\; \frac{c_m\,B}{A\sqrt{n}}.
\]

Finally, let $p_{\mathrm{eff}}\le p$ be the number of hidden units that actually affect the output, i.e., those with a nonzero row in $U$ and a nonzero entry in the $j^\star$-th row of $V$. The above bound holds with $p_{\mathrm{eff}}$ in place of $p$, hence $p_{\mathrm{eff}} \ge c_m B/(A\sqrt{n})$. Each such unit uses at least one nonzero parameter in $U$ and one in $V$, so the parameter counts $k$ satisfy $k \ge p_{\mathrm{eff}}$. Therefore
\[
k \;\ge\; \frac{c_m\,B}{A\sqrt{n}}
\;=\; \Omega\!\left(\frac{B}{A\sqrt{n}}\right),
\]
which proves the lemma.
\end{proof}

\begin{lemma}\label{lem:Hessian}
    There exists $R>0$ such that for every $i \in \{1,\dots,D\}$ and every $r<R$, 
there exist constants $\delta_i>0$ and $L_i>0$ with the following property:  
there exists a segment $G_i \subset B(x^{(i)},r)$ of length $\delta_i$ such that
\[
    |f_i(x)-f_i(y)| \;\ge\; L_i \,\|x-y\|_2,
    \qquad \forall\, x,y \in G_i .
\]
and moreover
\[
    f_i(x) > z_i, \qquad \forall\, x \in G_i ,
\]
\end{lemma}

\begin{proof}[Proof of Lemma~\ref{lem:Hessian}]
Fix $i\in\{1,\dots,D\}$ and denote $x^\star:=x^{(i)}$, $f:=f_i$ and
$H_\star:=\nabla_x^2 f(x^\star)$. By positive definiteness,
let $\lambda_i:=\lambda_{\min}(H_\star)>0$.
By continuity of $\nabla^2 f$, there exists $R_i^{\mathrm H}>0$ such that
\[
  \nabla^2 f(x) \;\succeq\; \tfrac{\lambda_i}{2}\, I
  \qquad \text{for all } x\in B(x^\star,R_i^{\mathrm H}).
\]
Set $\mu_i:=\lambda_i/2>0$.
Because the domain is $[0,1]^d$ and $x^\star\in[0,1]^d$, we could choose a unit vector
$v_i$ pointing strictly into the cube at $x^\star$ (if $x^\star$ is interior,
take any unit vector). Define
\[
  \tau_i \;:=\; \sup\{\, t>0 \;:\; x^\star + s v_i \in [0,1]^d \text{ for all } s\in[0,t] \,\} \;>\; 0,
\]
and set $R_i := \min\{R_i^{\mathrm H},\,\tau_i\}$. Take $R := \min_{1\le i\le D} R_i > 0$. Fix any $r\in(0,R)$ and consider the restriction
\[
  g(t) \;:=\; f(x^\star + t v_i), \qquad t\in[0,r].
\]
Then
\[
  g''(t) \;=\; v_i^\top \nabla^2 f(x^\star+t v_i)\,v_i \;\ge\; \mu_i \quad \text{for all } t\in[0,r].
\]

Since $x^\star$ minimizes $f$ on $[0,1]^d$ and $v_i$ is feasible inward,
we have $g(t)\ge g(0)$ for small $t\ge 0$ leading to the one‑sided derivative
$g'(0+)\ge 0$ (if $x^\star$ is interior then $\nabla f(x^\star)=0$ so $g'(0)=0$).
Because $g''\ge \mu_i$, the derivative $g'$ is increasing and thus
\[
  g'(t) \;\ge\; g'(0+) + \mu_i t \;\ge\; \mu_i t, \qquad t\in[0,r].
\]
Let $a:=r/4$ and $b:=r/2$ and define the segment
\[
  G_i \;:=\; \{\, x^\star + t v_i : t\in[a,b] \,\} \;\subset\; B(x^\star,r),
\]
whose length is $\delta_i := b-a = r/4$.
For any $x=x^\star+t v_i$ and $y=x^\star+s v_i$ in $G_i$ with $t>s$, the mean value theorem
gives some $\xi\in(s,t)\subset[a,b]$ such that
\[
  |f(x)-f(y)|
  \;=\; |g(t)-g(s)|
  \;=\; |g'(\xi)|\,|t-s|
  \;\ge\; \mu_i a\,|t-s|
  \;=\; \Bigl(\frac{\lambda_i r}{8}\Bigr)\,\|x-y\|_2.
\]
Therefore the choice
\[
  L_i \;:=\; \frac{\lambda_i r}{8} \;>\; 0
\]
works uniformly for all $x,y\in G_i$.
The segment $G_i$ does not contain $x^\star$, for all its points are at distance at least $a=r/4>0$ from $x^\star$.
By uniqueness of the minimizer, $f(x)>f(x^\star)=z_i$ for all $x\in G_i$.

The lemma holds with $\delta_i=r/4$ and
$L_i=(\lambda_{\min}(\nabla_x^2 f_i(x^{(i)}))\, r)/8$.
\end{proof}

\begin{lemma} \label{lem:gradient} 
    Let $(z_1,\dots,z_D)$ denote the minima defined above.  
Then there exist constants $r_0>0$ and $L_0>0$ such that the following holds:  
for any $i \in \{1,\dots,D\}$ and any perturbation $\delta_0$ with $|\delta_0|<r_0$, 
\[
    \bigl|F_0(z_1,\dots,z_i+\delta_0,\dots,z_D) - F_0(z_1,\dots,z_D)\bigr|
    \;\ge\; L_0\,|\delta_0| .
\]  
\end{lemma}

\begin{proof}[Proof of Lemma~\ref{lem:gradient}]

Denote $e_i$ for the $i$-th standard basis vector of $\mathbb{R}^D$.
By assumption, $m_i \coloneqq \bigl|\partial_i F_0(z)\bigr|>0$ for each $i$.
Since $F_0\in C^1$, the map $u\mapsto \partial_i F_0(u)$ is continuous at $z$.
Hence for each $i$, there exists $r_i^{\mathrm{cont}}>0$ such that
\[
\bigl|\partial_i F_0(u)\bigr|\;\ge\;\tfrac12\,m_i
\quad\text{whenever}\quad \|u-z\|_\infty<r_i^{\mathrm{cont}}.
\]
If necessary, shrink $r_i^{\mathrm{cont}}$ so that the line segment
$\{\,z+t e_i: |t|<r_i^{\mathrm{cont}}\,\}$ lies in $[0,1]^D$.
Define uniform constants
\[
L_0 \;\coloneqq\; \tfrac12 \min_{1\le i\le D} m_i \;>\;0,
\qquad
r_0 \;\coloneqq\; \min_{1\le i\le D} r_i^{\mathrm{cont}} \;>\;0 .
\]
Fix $i$ and $\delta_0$ with $|\delta_0|<r_0$.
Consider the one-dimensional slice $g_i(t)\coloneqq F_0(z+t e_i)$ for $|t|<r_0$.
Then $g_i$ is $C^1$ and $g_i'(t)=\partial_i F_0(z+t e_i)$.
By the mean value theorem, there exists $\theta\in(0,1)$ such that
\[
F_0(z+\delta_0 e_i)-F_0(z)=g_i'(\theta\delta_0)\,\delta_0
=\partial_i F_0\!\bigl(z+\theta\delta_0 e_i\bigr)\,\delta_0 .
\]
Taking absolute values and using the lower bound on
$\bigl|\partial_i F_0(\cdot)\bigr|$ inside the $\ell_\infty$-ball of radius $r_0$
around $z$, we have
\[
\bigl|F_0(z+\delta_0 e_i)-F_0(z)\bigr|
\;\ge\; L_0\,|\delta_0| ,
\]
which is the desired inequality.
\end{proof}

\begin{lemma} \label{lem:separation basin}
    Let $z_i=\min_{x\in[0,1]^d} f_i(x)$ and let $x^{(i)}$ denote the unique minimizer of $f_i$ (as assumed above). Then there exist constants $R_0>0$ and $\varepsilon_0>0$ such that:
\begin{enumerate}
    \item The open balls $\{B(x^{(i)},R_0)\}_{i=1}^D$ are pairwise disjoint.
    \item For each $i\in\{1,\dots,D\}$ and every $x\in[0,1]^d\setminus B(x^{(i)},R_0)$,
    \[
        f_i(x) \;>\; z_i + \varepsilon_0 .
    \]
\end{enumerate}
\end{lemma}

\begin{proof}[Proof of Lemma~\ref{lem:separation basin}]
  
Since the minimizers $\{x^{(i)}\}_{i=1}^D$ are pairwise distinct and finite in number, we have
\[
\Delta \;:=\; \min_{i\neq j}\,\bigl\|x^{(i)}-x^{(j)}\bigr\|_2 \;>\; 0 .
\]
Set $R_0 := \tfrac12 \Delta$. If $i\neq j$ and $x\in B(x^{(i)},R_0)$,  by the triangle inequality, we have
\[
\|x-x^{(j)}\|_2 \;\ge\; \|x^{(i)}-x^{(j)}\|_2 - \|x-x^{(i)}\|_2 \;>\; \Delta - R_0 \;=\; R_0,
\]
so $x\notin B(x^{(j)},R_0)$. Hence the balls are pairwise disjoint, proving the first part.

For the second part, fix $i$ and define the compact set $K_i := [0,1]^d \setminus B(x^{(i)},R_0)$. The continuity of $f_i$ implies that the minimum
\[
m_i \;:=\; \min_{x\in K_i} f_i(x)
\]
is attained on $K_i$. Because $x^{(i)}\notin K_i$ and $x^{(i)}$ is the unique global minimizer on $[0,1]^d$, we have $m_i>z_i$. Let $\varepsilon_i := m_i - z_i>0$ and set
\[
\varepsilon_0 \;:=\; \tfrac12 \min_{1\le i\le D} \varepsilon_i \;>\; 0.
\]
Then, for every $x\in K_i$, we have
\[
f_i(x) \;\ge\; m_i \;=\; z_i+\varepsilon_i \;\ge\; z_i + 2\varepsilon_0 \;>\; z_i+\varepsilon_0,
\]
Thus, we have proved this lemma.
\end{proof}

Before the proof, We also provide a notation table to help with understanding in Table~\ref{tab:notation-flow-arrows}. 
\begin{table}[t]
\centering
\begin{tabular}{ll}
\hline
\textbf{Notation flow (dependency structure)} & \textbf{Meaning} \\
\hline

$x^{(i)}$ & Point where $f_i$ achieves minimum \\
\quad $\rightarrow B(x^{(i)},r)$ & Basin region for retrieval coordinate $i$ \\
 & (Basin around $x^{(i)}$)\\
\quad\quad $\rightarrow G_i, K_i$ & Monotone local segment near $x^{(i)}$ \\
& (In $B(x^{(i)},r)$) \\
\\[-4pt]
$P_0$ & The set of all candidate points. \\
 & (We only choose $x_t \in P_0$)\\
$S_i$ & Index partition for retrieval coordinate $i$ \\
 & ($i=1, \dots, D$)\\
\\[-4pt]

Attention head $j$ & Defines response at position $t$ \\
\quad $\rightarrow \lambda_j(x,t)$ & Attention score \\
\quad \quad $\rightarrow (y_j,t_j)$ & Maximum‐attention point selected by \\
 & head $j$ in $P_0\times S_j$\\
\quad \quad $\rightarrow Y=\{y_1,\dots,y_s\}$ & Chosen maximizers of attention score \\
 & (one per head) \\
\quad $\rightarrow v_j(x,t)$ & Value embedding \\
\\[-4pt]

WLOG, suppose $K_1 \cap Y = \emptyset $. &  \\
\quad $\rightarrow T_0$ & $T_0 \subset S_1$, indices not in $(y_j, t_j)$, $j=1, \dots, s$ \\
\quad\quad $\rightarrow \eta:[0,1]\to G_1$ & Coordinate system on $G_1$ \\
\quad\quad\quad $\rightarrow q=f_1\circ\eta$ & Rewriting $f_1|_{G_1}$ into the coordinate system. \\
\quad\quad $\rightarrow U_t$ & Discrete grid on $[0,1]$ at index $t$ \\
\quad\quad\quad $\rightarrow z_\ell(t)$ & Candidate point for subsequence $\ell$, $z_\ell(t) \in \eta(U_t)$\\
\\[-4pt]

Adversarial subsequences & \\
\quad $\rightarrow Z_\ell = (z_\ell(1),\dots,z_\ell(T_0))$ & Two subsequences almost  \\
 & indistinguishable by attention head.\\
\quad\quad $\rightarrow W_\ell$ & Full sequence embedding $Z_\ell$ \\
\quad\quad\quad $\rightarrow w_\ell(t)$ & Token of $W_\ell$ of index $t$ \\
\quad\quad\quad $\rightarrow I_1,I_2,I_3$ & Partition of indices: differ / agree / remaining \\
\\[-4pt]

Per-head analysis & \\
\quad $\rightarrow Q_{j,i}$ & Attention mass on $I_j$ ($j\in\{1,2,3\}, i\in \{1, \dots, s\}$) \\
\quad $\rightarrow V_{j,i}$ & Weighted value average on $I_j$ \\
\\[-4pt]

\hline
\end{tabular}
\caption{Flow-style dependency map of notation introduced in the proof of Theorem~2.2.}
\label{tab:notation-flow-arrows}
\end{table}

\begin{proof}[Proof of Theorem~\ref{thm2:p2}]
Given the target function under the assumptions. For any given single-layer transformer 
defined in the main context, our goal is to find two different sequences such that their output in the part  
\begin{align}
    \operatorname{Concat}_{i=1}^{h}
        \Bigl(
            \sum_{t=1}^{T}
            \sigma\!\left[
                (W_{K,i}\hat{x}(t))^\top W_{Q,i}\hat{c}_0
            \right]
            W_{V,i}\hat{x}(t)\Bigr)
\end{align}
are very close (differs by only $O(\epsilon^{k+1})$), but their output from the target function differs by at least $3 \epsilon$, then according to lemma~\ref{lem:FFN-rate}, we have the required parameter count for the FFN $\hat{F}$ to be at least $\Omega(1/\epsilon^k)$. 
\paragraph{Notations}
For each head $i=1,\dots,s$, define the attention weight function  
\[
    \lambda_i(x,t) \;=\; 
    \exp\!\bigl(\gamma\, (W_{K,i} P_\phi(x,t))^\top W_{Q,i}\hat{c}_0 \bigr),
\]
and the value mapping  
\[
    v_i(x,t) \;=\; W_{V,i} P_\phi(x,t) \;\in\; \mathbb{R}^n,
\]
where $\gamma>0$ is the softmax scaling factor, 
$W_{Q,i},W_{K,i}\in\mathbb{R}^{n\times E}$ are the query and key projections, 
and $W_{V,i}\in\mathbb{R}^{n\times E}$ is the value projection for head $i$.\\
\paragraph{Notation of sets}
Without loss of generality, we assume $x^{(i)}$ belongs to the interior of $[0,1]^d$, and the other case can be treated with the same method below. 
From lemma~\ref{lem:Hessian}, lemma~\ref{lem:gradient} and lemma~\ref{lem:separation basin} we have that there exists $R>0$ and segments $G_i \subset B(x^{(i)}, R), i = 1, \dots, D$ and $L, \delta_0, r >0$ satisfying the following: 
\begin{itemize}
    \item $\forall i $, $\forall x, y \in G_i$, we have $|f_i(x) - f_i(y)| > L\|x-y\|_2$.
    \item $\forall j \neq i$ and $\forall x \in B(x^{(j)}, R), y \in B(x^{(i)},R)$, we have $f_i(y) - f_i(x) > \delta_0$.
    \item The length of $G_i$ is $r$, $\forall i= 1, \dots, D$.
    \item For any $i \in \{1,\dots,D\}$ and any perturbation $\delta_1$ with $|\delta_1|<\max_{x \in B(x^{(i)}, R)} (f_i(x) - z_i)$, 
\[
    \bigl|F_0(z_1,\dots,z_i+\delta_1,\dots,z_D) - F_0(z_1,\dots,z_D)\bigr|
    \;\ge\; L\,|\delta_1| .
\]  
\end{itemize}
We denote by $K_i := G_i \cup \{x^{(i)}\}, i = 1, \dots, D$, and $P_0 = \cup_{i=1}^{D}K_i$. 
Recall that $k = \frac{\frac{1}{4}T - s-D +1}{(n+1)s +1} - 1$. We assume without loss of generality that $k >0$ and $\frac{1}{4}T -s -D +1 >0$, otherwise the result would be trivial. 
\paragraph{Max weight for each head}
For $j=1,\dots,s$, define recursively the pairs $(y_j,t_j)$ as follows:  
\begin{itemize}
    \item For the first head,
    \[
        (y_1,t_1) \;=\; \arg\max_{\substack{y \in P_0 \\ t \in S_1}}
        \,\lambda_1(y,t).
    \]
    \item For $j>1$,
    \[
        (y_j,t_j) \;=\; \arg\max_{\substack{y \in P_0 \\ t \in S_j \\ t \notin \{t_1,\dots,t_{j-1}\}}}
        \,\lambda_j(y,t).
    \]
    \item If maximum can be obtained at multiple $(y, t)$, then choose one of them. 
\end{itemize}
Let $Y = \{y_1,\dots,y_s\}$.  
Since the sets $K_1,\dots,K_D$ are pairwise disjoint and $s<D$, 
there exists at least one index $i \in \{1,\dots,D\}$ such that
\[
    K_i \,\cap\, Y \;=\; \varnothing.
\]
Without loss of generality, we assume that $i=1$. As we have $|S_i| \ge \frac{1}{4}T > s+D-1, i=1, \dots, D$, we have that there exists a set of $(t^*_2, \dots, t^*_{D})$ such that 
\begin{itemize}
    \item $t^*_{j} \notin \{t_1, \dots, t_s\}$, for $j=2, \dots, D$.
    \item $t^*_{j}$ are pairwise distinct.
    \item $t^*_{j} \in S_j, j=2, \dots, D$.
\end{itemize}
Let $T_0 = \frac{1}{4} T - s-D+1 >0$ and assume that $T_0$ is a integer. Then we have $|S_1- \{t_1, \dots, t_s, t^*_2, \dots, t^*_D\}| \ge T_0 >0$. Without loss of generality, suppose $\{1, 2, \dots, T_0\} \subset S_1- \{t_1, \dots, t_s, t^*_2, \dots, t^*_D\}$.
\paragraph{Sequences to be considered}
As $G_1$ is a segment of length $r$, then it is natural to assign coordinate system $\eta : [0,1] \to G_1 $ on $G_1$, with $q:=f_1\circ \eta$ being a monotonically increasing function on $[0,1]$. The monotone property is as a result of $|f_1(x) - f_1(y)| \ge L\|x-y\|_2$. \\
We denote by $M = T_0\lfloor \frac{rL^2}{3T_0\epsilon} \rfloor$. As $T_0 | M$, Construct the following $T_0$ sets:
\begin{align}
    U_j = \frac{j-1}{T_0} + \{\frac{1}{M}, \dots, \frac{(\frac{T_0}{M})}{M}\}, j=1, \dots, T_0
\end{align}
We have $|U_j| = \lfloor \frac{rL^2}{3T_0\epsilon} \rfloor = O(1/\epsilon)$. 
\begin{claim}{Existence of two distinct sub-sequence} \label{claim:subsequence}\\
    There exists two subsequences $z_1(1), \dots, z_1(T_0)$ and $z_2(1), \dots z_2(T_0)$ with $z_i(t) \in \eta(U_t)$ satisfying the following conditions. 
    \begin{itemize} \label{item:condition}
    \item $\left\|\frac{\sum_{t=1}^{T_0}\lambda_i(z_1(t), t)v_i(z_1(t),t) }{\sum_{t=1}^{T_0}\lambda_i(z_1(t), t)}
    - \frac{\sum_{t=1}^{T_0}\lambda_i(z_2(t), t)v_i(z_2(t),t) }{\sum_{t=1}^{T_0}\lambda_i(z_2(t), t)}\right\|_2
    \le \frac{\epsilon^{k+1}}{3T_0}$, for $i=1,\dots,s$.

    For each $i=1, \dots, s$, either of the following holds:
    \begin{enumerate}
        \item $\frac{\sum_{t=1}^{T_0}\lambda_i(z_1(t), t)}{\sum_{t=1}^{T_0}\lambda_i(z_2(t), t)} 
        \in \left[1/ (1+\tfrac{\epsilon^{k+1}}{12T_0^2}), \, 1+\tfrac{\epsilon^{k+1}}{12T_0^2}\right]$. 

        \item $\max_{j=1,2}\sum_{t=1}^{T_0}\lambda_i(z_j(t), t) 
        \le \tfrac{\epsilon^{k+1}}{4} \sum_{w=1}^s\lambda_{i}(y_w, t_w)$.
    \end{enumerate}
\end{itemize}
\end{claim}
\begin{proof}
We compare the orders of $1/\epsilon$ appearing on both sides of the conditions.  

First, since $|U_t| = O(1/\epsilon)$ for each $t$, the total number of possible 
choices of subsequences $(z(1),\dots,z(T_0))$ is at most $O(1/\epsilon^{T_0})$.  

Next, to satisfy condition (1), note that both vectors involved are 
$n$-dimensional with norms bounded by $1$. Thus, the discretization required to 
achieve accuracy $\epsilon^{k+1}/(3T_0)$ in the $\ell_2$ norm leads to at most 
$O(1/\epsilon^{(k+1)ns})$ distinct possibilities, since there are $s$ heads.  

For condition (2), observe that
\[
    \sum_{w=1}^s \lambda_i(y_w,t_w) 
    \;\ge\; \tfrac{1}{T_0}\max_{j=1,2}\sum_{t=1}^{T_0}\lambda_i(z_j(t),t).
\]
Hence, for each $i$, the relevant interval can be partitioned into at most 
$O\!\left(\tfrac{-\log \epsilon}{\epsilon^{k+1}}\right)$ sub-intervals. 
Taken across $s$ heads, this contributes at most 
$O\!\left((-\log\epsilon)^s / \epsilon^{(k+1)s}\right)$ possibilities.  

Combining the two conditions, the total number of distinct admissible cases is 
bounded above by
\[
    O\!\left(\frac{(-\log\epsilon)^s}{\epsilon^{(k+1)ns+(k+1)s}}\right).
\]
Since $T_0 \ge (k+1)ns + (k+1)s + 1$, we have
\[
    O\!\left(\frac{1}{\epsilon^{T_0}}\right)
    \;\gg\;
    O\!\left(\frac{(-\log\epsilon)^s}{\epsilon^{(k+1)ns+(k+1)s}}\right).
\]
Therefore, by the pigeonhole principle, there must exist two distinct subsequences 
$(z_1(1),\dots,z_1(T_0))$ and $(z_2(1),\dots,z_2(T_0))$ satisfying all the 
conditions of Claim~\ref{claim:subsequence}.
\end{proof}
\paragraph{Construction of Distinct sequences}
From Claim~\ref{claim:subsequence}, we have constructed two sub-sequences $Z_1, Z_2$ satisfying the given conditions. We now consider the construction of two full input sequence $W_1, W_2$:
\begin{itemize}
    \item For $t=1, \dots, T_0$, if $z_1(t) = z_2(t)$, then $w_1(t) = w_2(t) = x^{(D)}$. Otherwise, $w_1(t) = z_1(t), w_2(t)=z_2(t)$. 
    \item $w_j(t_i) = y_i$, $i=1, \dots, s;\quad j=1,2$.
    \item $w_j(t^*_i) = x^{(i)}$, $i = 2, \dots, D; \quad j =1, 2$.
    \item For all other $t$, $w_j(t) = x^{(D)}$. 
\end{itemize}

\paragraph{Difference of $W1,W2$ applied to target function} \label{claim:difference}
Denote by $I_1$ the set of all indices $t$ with $z_1(t) \neq z_2(t)$, and $I_2 = [T_0] - I_1$, $I_3 = [T] - I_1$. It is clear from the difference of $Z_1, Z_2$ that $I_1 \neq \varnothing$. \\
We then define the following notations for the simplicity of calculation (defined for each head $i=1, \dots, s$):
\begin{itemize}
    \item $Q_{1,i} = \sum_{t \in I_1} \lambda_i(w_1(t), t)$.
    \item $Q_{2,i}= \sum_{t \in I_1} \lambda_i(w_2(t), t)$.
    \item $V_{1,i} = (\sum_{t \in I_1} \lambda_i(w_1(t), t) v_i(w_1(t), t))/Q_{1,i}$.
    \item $V_{2,i} = (\sum_{t \in I_1} \lambda_i(w_2(t), t) v_i(w_2(t), t))/Q_{2,i}$.
    
    \item $Q_{3,i} = \sum_{t \in I_2} \lambda_i(z_1(t), t)$, which is also the same if defined on $Z_2$.
    \item $V_{3,i} = (\sum_{t \in I_2} \lambda_i(z_1(t), t) v_i(z_1(t), t))/Q_{3,i}$, which is the same if defined on $Z_2$.
    \item $Q_{4, i} = \sum_{t \in I_3} \lambda_i(w_1(t), t)$, which is the same if defined on $W_2$.
    \item $V_{4, i} = (\sum_{t \in I_3} \lambda_i(w_1(t), t) v_i(w_1(t), t))/Q_{4,i}$, which is the same if defined on $W_2$.
\end{itemize}
As $\lambda_i()$ maps to positive values, $V_{j,i}$ are convex combinations of $v_i()$, whose norm is bounded by $1$ according to the constraint section~\ref{assump:setup}. Therefore $\|V_{j,i}\| \le 1$, $j=1,2,3,4$.

As $f_1 \circ \eta$ is monotone on $[0,1]$, let $\tilde{t} = \max_{t\in I_1}t$, then we have 
\begin{itemize}
    \item $\max_{t \in S_1} f_1(w_1(t)) = f_1(w_1(\tilde{t}))$.
    \item $\max_{t \in S_1} f_1(w_2(t)) = f_1(w_2(\tilde{t}))$.
\end{itemize}
And by construction we know that 
\begin{align}
    \|(w_1(\tilde{t})) - (w_2(\tilde{t}))\| \ge \frac{r}{M}
\end{align}
which is the minimal distance for any two points in $U_{\tilde{t}}$. Then we have
\begin{align}
    |f_1(w_1(\tilde{t})) - f_1(w_2(\tilde{t}))| \ge \frac{rL}{M}
\end{align}
As we have for $i=2, \dots, D$
\begin{itemize}
    \item $\max_{t \in S_i} f_i(w_1(t)) = z^{(i)}$.
    \item $\max_{t \in S_i} f_i(w_2(t)) = z^{(i)}$.
\end{itemize}
Then following the perturbation property of $F_0$ defined above we have that the difference of output between the two sequence to be at least $\frac{rL^2}{M}$, which is greater than $3\epsilon$. Then $\epsilon$-approximation requires that $|Model(W_1) - Model(W_2)| \ge \epsilon$. 
\paragraph{$W_1$ and $W_2$ are close after attention layer}
For any given head $i$, we consider the the two cases given in \ref{item:condition}.
\paragraph{Case 1}
Case 1 can be rewritten as follows:
\begin{itemize}
    \item $\|\frac{Q_{1,i}V_{1,i} + Q_{3,i}V_{3,i}}{Q_{1,i}+ Q_{3,i}}- \frac{Q_{2,i}V_{2,i} + Q_{3,i}V_{3,i}}{Q_{2,i}+ Q_{3,i}}\|_2 \le \frac{\epsilon^{k+1}}{3T_0}$.
    \item $\frac{Q_{1,i}+ Q_{3,i}}{Q_{2,i}+ Q_{3,i}} \in \left[1/ (1+\tfrac{\epsilon^{k+1}}{12T_0^2}), \, 1+\tfrac{\epsilon^{k+1}}{12T_0^2}\right]$.
\end{itemize}
Without loss of generality, we assume $Q_{1, i} \ge Q_{2,i}$.
By calculation, we have 
\begin{align}\label{eq:21}
    &\frac{Q_{1,i}V_{1,i} + Q_{3,i}V_{3,i}}{Q_{1,i}+ Q_{3,i}}- \frac{Q_{2,i}V_{2,i} + Q_{3,i}V_{3,i}}{Q_{2,i}+ Q_{3,i}}\\
    =& \frac{Q_{3,i}(Q_{2,i} - Q_{1,i})(V_{3,i} - V_{2,i})}{(Q_{1,i}+Q_{3,i})(Q_{2,i}+Q_{3,i}) } + \frac{Q_{1,i}}{Q_{1,i}+Q_{3,i}}(V_{1,i} - V_{2,i}).
\end{align}
We have already known that $Q_{4,i} \ge \frac{Q_{1,i} + Q_{2,i}}{T_0}$ (As $Q_{4,i}$ has the max weight of each head in it). Then 
\begin{align}\label{ineq:21}
    &\|\frac{Q_{4,i}(Q_{2,i} - Q_{1,i})(V_{4,i} - V_{2,i})}{(Q_{1,i}+Q_{4,i})(Q_{2,i}+Q_{4,i}) }\| \\
    \le & \|\frac{(Q_{2,i} - Q_{1,i})(V_{4,i} - V_{2,i})}{(Q_{1,i}+Q_{4,i}) }\| \\
    \le & \|\frac{T_0 (Q_{2,i} - Q_{1,i})(V_{4,i} - V_{2,i})}{(Q_{1,i}+Q_{3,i}) }\| \\
    \le & \|\frac{T_0 \epsilon^{k+1}(V_{4,i} - V_{2,i})}{12T_0^2 }\|\\
    \le & \frac{\epsilon^{k+1}}{6T_0}
\end{align}
Similarly, we also have 
\begin{align}\label{ineq:22}
    \|\frac{Q_{3,i}(Q_{2,i} - Q_{1,i})(V_{3,i} - V_{2,i})}{(Q_{1,i}+Q_{3,i})(Q_{2,i}+Q_{3,i}) }\| \le \frac{\epsilon^{k+1}}{6T_0}
\end{align}
From inequality~\ref{ineq:21} and substituting equation~\ref{eq:21}, we have 
\begin{align}
    \|\frac{Q_{1,i}}{Q_{1,i}+Q_{3,i}}(V_{1,i} - V_{2,i})\| \le \frac{\epsilon^{k+1}}{6T_0} + \frac{\epsilon^{k+1}}{3T_0} = \frac{\epsilon^{k+1}}{2T_0}
\end{align}
Therefore
\begin{align}
    &\|\frac{Q_{1,i}}{Q_{1,i}+Q_{4,i}}(V_{1,i} - V_{2,i})\| \\
    \le &\|\frac{T_0Q_{1,i}}{Q_{1,i}+Q_{3,i}}(V_{1,i} - V_{2,i})\| \\
    \le & \frac{\epsilon^{k+1}}{2}
\end{align}
Thus 
\begin{align}
    &\|\frac{Q_{1,i}V_{1,i} + Q_{4,i}V_{4,i}}{Q_{1,i}+ Q_{4,i}}- \frac{Q_{2,i}V_{2,i} + Q_{4,i}V_{4,i}}{Q_{2,i}+ Q_{4,i}}\|_2 \\
    \le & \|\frac{Q_{4,i}(Q_{2,i} - Q_{1,i})(V_{4,i} - V_{2,i})}{(Q_{1,i}+Q_{4,i})(Q_{2,i}+Q_{4,i}) }\| + \|\frac{Q_{1,i}}{Q_{1,i}+Q_{4,i}}(V_{1,i} - V_{2,i})\| \\
    \le &  \frac{\epsilon^{k+1}}{6T_0} + \frac{\epsilon^{k+1}}{2} \\
    \le & \epsilon^{k+1}
\end{align}
\paragraph{Case 2}
Case 2 can be rewritten as follows:
\begin{itemize}
    \item $\|\frac{Q_{1,i}V_{1,i} + Q_{3,i}V_{3,i}}{Q_{1,i}+ Q_{3,i}}- \frac{Q_{2,i}V_{2,i} + Q_{3,i}V_{3,i}}{Q_{2,i}+ Q_{3,i}}\|_2 \le \frac{\epsilon^{k+1}}{3T_0}$.
    \item $Q_{1,i}+ Q_{3,i} \le \frac{\epsilon^{k+1}}{4} \sum_{w=1}^s\lambda_i(y_w, t_w) \le \frac{\epsilon^{k+1}}{4}Q_{4,i}$.
    \item $Q_{2,i}+ Q_{3,i} \le \frac{\epsilon^{k+1}}{4} \sum_{w=1}^s\lambda_i(y_w, t_w) \le \frac{\epsilon^{k+1}}{4}Q_{4,i}$.
\end{itemize}
Thus 
\begin{align}
    &\|\frac{Q_{1,i}V_{1,i} + Q_{4,i}V_{4,i}}{Q_{1,i}+ Q_{4,i}}- \frac{Q_{2,i}V_{2,i} + Q_{4,i}V_{4,i}}{Q_{2,i}+ Q_{4,i}}\|_2 \\
    \le & \|\frac{Q_{4,i}(Q_{2,i} - Q_{1,i})(V_{4,i} - V_{2,i})}{(Q_{1,i}+Q_{4,i})(Q_{2,i}+Q_{4,i}) }\| + \|\frac{Q_{1,i}}{Q_{1,i}+Q_{4,i}}(V_{1,i} - V_{2,i})\| \\
    \le &  \|\frac{(Q_{2,i} - Q_{1,i})(V_{4,i} - V_{2,i})}{(Q_{2,i}+Q_{4,i}) }\| + \|\frac{Q_{1,i}}{Q_{1,i}+Q_{4,i}}(V_{1,i} - V_{2,i})\| \\
    \le& \|\frac{\epsilon^{k+1}(V_{4,i} - V_{2,i})}{4 }\| + \|\frac{\epsilon^{k+1}}{4}(V_{1,i} - V_{2,i})\|\\
    \le & \epsilon^{k+1}
\end{align}

And it can be seen from definition that 
\begin{itemize}
    \item $\frac{Q_{1,i}V_{1,i} + Q_{4,i}V_{4,i}}{Q_{1,i}+ Q_{4,i}}$ is the output of the $i$-th head of the attention layer with input sequence $W_1$. (which means $\frac{Q_{1,i}V_{1,i} + Q_{4,i}V_{4,i}}{Q_{1,i}+ Q_{4,i}} = \sum_{t=1}^{T}
            \sigma\!\left[
                (W_{K,i}\hat{w_1}(t))^\top W_{Q,i}\hat{c}_0
            \right]
            W_{V,i}\hat{w_1}(t)$).
    \item $\frac{Q_{2,i}V_{1,i} + Q_{4,i}V_{4,i}}{Q_{2,i}+ Q_{4,i}}$ is the output of the $i$-th head of the attention layer with input sequence $W_2$. (which means $\frac{Q_{2,i}V_{2,i} + Q_{4,i}V_{4,i}}{Q_{2,i}+ Q_{4,i}} = \sum_{t=1}^{T}
            \sigma\!\left[
                (W_{K,i}\hat{w_2}(t))^\top W_{Q,i}\hat{c}_0
            \right]
            W_{V,i}\hat{w_2}(t)$).
\end{itemize}
Then for each $i=1, \dots, s$, we have that 
\begin{align}
    \|\sum_{t=1}^{T}
            \sigma\!\left[
                (W_{K,i}\hat{w_1}(t))^\top W_{Q,i}\hat{c}_0
            \right]
            W_{V,i}\hat{w_1}(t) - \sum_{t=1}^{T}
            \sigma\!\left[
                (W_{K,i}\hat{w_2}(t))^\top W_{Q,i}\hat{c}_0
            \right]
            W_{V,i}\hat{w_2}(t)\| \le \epsilon^{k+1} 
\end{align}
Therefore, as $W_O$ have entries bounded by $1$, we have
\begin{align}
    & \|\Bigl[\hat{c}_0 + W_O\operatorname{Concat}_{i=1}^{h}
        \Bigl(
            \sum_{t=1}^{T}
            \sigma\!\left[
                (W_{K,i}\hat{w_1}(t))^\top W_{Q,i}\hat{c}_0
            \right]
            W_{V,i}\hat{w_1}(t)\Bigr) \Bigr]\\
    -&  \Bigl[ \hat{c}_0 + W_O\operatorname{Concat}_{i=1}^{h}
        \Bigl(
            \sum_{t=1}^{T}
            \sigma\!\left[
                (W_{K,i}\hat{w_2}(t))^\top W_{Q,i}\hat{c}_0
            \right]
            W_{V,i}\hat{w_2}(t)\Bigr)\Bigr]\| \\
    \le & s\epsilon^{k+1}
\end{align}
However, it has been proven above that we need $|Model(W_1) - Model(W_2)| \ge \epsilon$ to achieve $\epsilon$-approximation of the target function. According to lemma~\ref{lem:FFN-rate}, the required parameter count of the FFN $\hat{F}$ is of order $\Omega(\epsilon/\epsilon^{k+1})$. 
Thus the parameter count required to achieve $\epsilon$-approximation is $\Omega(1/\epsilon^k)$.

\end{proof}

\begin{remark}\label{rem:lower-bounds}\textbf{Tightness of Theorem~\ref{thm2:p2}}
The lower bound in Theorem~\ref{thm2:p2} remains essentially tight under 
several relaxations of the feed-forward block $\hat{F}$.  
If $\hat{F}$ uses Heaviside activations instead of $1$-Lipschitz activations, 
matching upper bounds can be constructed, but this case is impractical since 
Heaviside activations are rarely used in practice. If parameter norms are permitted to scale as $O(T^{1/\epsilon})$, 
the parameter count can be reduced to $O(1/\epsilon^{\gamma+1})$, 
though this scenario is likewise unrealistic in practical settings. Finally, if $\hat{F}$ is allowed up to five layers, the lower bound changes to 
$\Omega(1/\epsilon^{k/4})$, which does not alter the qualitative conclusion. 
\end{remark}

\subsection{Proof of Theorem~\ref{thm2:p3}}
\paragraph{Proof Sketch.}  
The argument is based on an explicit construction. 
We begin with trivial attention, so that the post-attention output is simply the 
averaged concatenation $\tfrac{1}{T}(x(1),\dots,x(T)) \in \mathbb{R}^{Td}$. 
The feed-forward block can then be used to compute the transformations 
$f_i(x(t))$, perform the necessary comparisons, and approximate $F_0$, 
as ensured by Lemmas~\ref{lem:relu-max} and~\ref{lem:stacking}.  

Having outlined the main idea, we now proceed to the detailed proof. 
As a first step, we introduce several auxiliary lemmas that will be used in the argument.

\begin{lemma}\label{lem:stacking}
    Fix a pointwise activation $\sigma$ (e.g., ReLU or any activation used in this paper).
Let $F_1:\mathbb{R}^{m_1}\to\mathbb{R}^{m_2}$ be a $2$-layer fully connected network,
$F_2:\mathbb{R}^{m_2}\to\mathbb{R}^{m_3}$ a $3$-layer fully connected network,
and $F_3:\mathbb{R}^{m_3}\to\mathbb{R}$ a $2$-layer fully connected network.
Let $W_1,W_2,W_3$ denote their respective (maximum) hidden widths, and set
$W:=\max\{W_1,W_2,W_3\}$.
Then there exists a $5$-layer fully connected network
$G:\mathbb{R}^{m_1}\to\mathbb{R}$ with activation $\sigma$ and hidden width at most $W$
such that
\[
    G(x)\;=\;F_3\!\bigl(F_2\!\bigl(F_1(x)\bigr)\bigr)\qquad \text{for all }x\in\mathbb{R}^{m_1}.
\]
\end{lemma}
\begin{proof}{Proof of Lemma~\ref{lem:stacking}}
    Write the three networks in affine–nonlinearity form (with a pointwise activation $\sigma$):
\[
\begin{aligned}
F_1(x) &= A_1\,\sigma(B_1 x + b_1) + a_1, && x\in\mathbb{R}^{m_1},\; F_1(x)\in\mathbb{R}^{m_2},\\
F_2(u) &= C_2\,\sigma\!\bigl(D_2\,\sigma(E_2 u + e_2)+ d_2\bigr)+ c_2, && u\in\mathbb{R}^{m_2},\; F_2(u)\in\mathbb{R}^{m_3},\\
F_3(v) &= p_3\,\sigma(Q_3 v + q_3) + r_3, && v\in\mathbb{R}^{m_3},\; F_3(v)\in\mathbb{R}.
\end{aligned}
\]
Define a $5$-layer fully connected network $G:\mathbb{R}^{m_1}\to\mathbb{R}$ by stacking the
hidden layers of $F_1$ (one), $F_2$ (two), and $F_3$ (one), keeping their original widths:
\[
\begin{aligned}
h_1(x) &:= \sigma(B_1 x + b_1),\\
u(x)   &:= A_1 h_1(x) + a_1,\\
h_2(x) &:= \sigma(E_2 u(x) + e_2),\\
h_3(x) &:= \sigma(D_2 h_2(x) + d_2),\\
v(x)   &:= C_2 h_3(x) + c_2,\\
h_4(x) &:= \sigma(Q_3 v(x) + q_3),\\
G(x)   &:= p_3 h_4(x) + r_3.
\end{aligned}
\]
By construction,
\[
G(x)= p_3\,\sigma\!\Bigl(Q_3\bigl(C_2\,\sigma(D_2\,\sigma(E_2(A_1\,\sigma(B_1 x + b_1)+a_1)+e_2)+d_2)+c_2\bigr)+q_3\Bigr)+r_3
= F_3\!\bigl(F_2\!\bigl(F_1(x)\bigr)\bigr).
\]
Thus $G$ realizes the composition exactly, has $4$ hidden layers (hence $5$ layers total),
and its hidden widths are precisely those of the constituent hidden layers of $F_1$, $F_2$, and $F_3$.
\end{proof}

\begin{lemma}[Approximating $\max$ with a shallow ReLU network]\label{lem:relu-max}
Let $f:[0,1]^T\to\mathbb{R}$ be $f(x_1,\dots,x_T)=\max\{x_1,\dots,x_T\}$. 
For any $\epsilon\in(0,1]$, there exists a fully connected ReLU network $\hat{f}$ with 
\emph{three layers} (i.e., two hidden layers and one output layer), whose hidden-layer widths are each at most 
$2T\,\lceil 1/\epsilon\rceil$, such that
$\hat{f}$ $\epsilon$-approximates $f$. 
\end{lemma}

\begin{proof}{Proof of Lemma~\ref{lem:relu-max}}
    Let \[
n = \lceil 1/\epsilon \rceil.
\]

For each coordinate $t \in [T]$ and each grid index $i = 0,1,\dots,n-1$, 
define the first hidden layer neurons by
\[
h_1(t,i) \;=\; \operatorname{ReLU}\!\left( x_t - \tfrac{i}{n} \right).
\]
For each $j = 0,1,\dots,n-1$, define the second hidden layer neurons by
\[
h_2(j) \;=\; \operatorname{ReLU}\!\left(\sum_{t=1}^T h_1(t,j)\right)
- \operatorname{ReLU}\!\left(\sum_{t=1}^T h_1(t,j) - \tfrac{1}{n}\right).
\]
Finally, the output of the network is given by
\[
\hat{f}(x_1,\dots,x_T) \;=\; \sum_{j=0}^{n-1} h_2(j).
\]
\begin{proof}[Claim]
Fix $j\in\{0,\dots,n-1\}$ and set
\[
S_j \;=\; \sum_{t=1}^T h_1(t,j)
       \;=\; \sum_{t=1}^T \operatorname{ReLU}\!\Big(x_t - \frac{j}{n}\Big).
\]
By definition,
\[
h_2(j) \;=\; \operatorname{ReLU}(S_j) \;-\; \operatorname{ReLU}\!\Big(S_j - \frac{1}{n}\Big).
\]

1) If $h_2(j)>0$, then necessarily $S_j>0$ (since $\operatorname{ReLU}(z)>0$ iff $z>0$), hence there exists some $t$ with 
\[
\operatorname{ReLU}\!\Big(x_t-\tfrac{j}{n}\Big) > 0
\quad\Longleftrightarrow\quad
x_t>\tfrac{j}{n}.
\]
Thus $h_2(j)>0$ only if $\exists\,t$ with $x_t>j/n$.

2) If there exists $t$ with $x_t>(j+1)/n$, then
\[
S_j \;\ge\; \operatorname{ReLU}\!\Big(x_t-\tfrac{j}{n}\Big)
\;>\; \tfrac{1}{n}.
\]
Therefore $S_j\ge \tfrac{1}{n}$, and we get
\[
h_2(j)\;=\; S_j - \Big(S_j - \tfrac{1}{n}\Big) \;=\; \tfrac{1}{n}.
\]
\end{proof}

Fix $x\in[0,1]^T$ and let $j$ be such that $\max_t x_t \in (j/n,(j+1)/n]$.
By construction,
\[
h_2(k)=0 \quad\text{for } k\ge j+1,\qquad
h_2(k)=\tfrac{1}{n} \quad\text{for } k\le j-1,
\]
and for $k=j$ we have
\[
0 \;\le\; h_2(j)
= \operatorname{ReLU}(S_j)-\operatorname{ReLU}\!\Big(S_j-\tfrac{1}{n}\Big)
\;\le\; \tfrac{1}{n},
\quad S_j:=\sum_{t=1}^T \operatorname{ReLU}\!\Big(x_t-\tfrac{j}{n}\Big).
\]
Hence
\[
\hat f(x)=\sum_{k=0}^{n-1} h_2(k)
= \sum_{k=0}^{j-1}\tfrac{1}{n} + h_2(j)
\in \Big[\tfrac{j}{n},\,\tfrac{j}{n}+\tfrac{1}{n}\Big]
= \Big[\tfrac{j}{n},\,\tfrac{j+1}{n}\Big].
\]
Since $\max_t x_t \in (j/n,(j+1)/n]$, it follows that
\[
0 \;\le\; |\hat f(x) - \max_t x_t| \;\le\; \tfrac{1}{n} \;\le\; \epsilon.
\]
Therefore $\hat f$ $\epsilon$-approximates $f(x)=\max_{t}x_t$ on $[0,1]^T$.
\end{proof}
\begin{proof}{Theorem~\ref{thm2:p3}}\\
We begin by fixing the embedding with positional information.  
Let $P_\phi:[0,1]^d \times [T]\to \mathbb{R}^{dT}$ be defined by
\[
    P_\phi(x(t),t) = (0,\dots,0,\,x(t),\,0,\dots,0),
\]
where the vector $x(t)$ occupies the $t$-th block of dimension $d$, and all 
other blocks are zero. With the classification token $\hat{c}_0=0$, the 
attention layer reduces to a trivial aggregation, and the output (prior to the 
feed-forward network) is
\[
    \tfrac{1}{T}\,(x(1),\dots,x(T)) \in [0,1]^{dT}.
\]
Given a target accuracy $\epsilon>0$, we construct three feed-forward networks 
$F_1,F_2,F_3$ as follows.  

\paragraph{Step 1: Approximating the component functions.}  
Define 
\[
    F_1 : \tfrac{1}{T}[0,1]^{d\times T} \;\to\; \mathbb{R}^{D\times T}, 
    \qquad 
    F_1\bigl(\tfrac{1}{T}x(1),\dots,\tfrac{1}{T}x(T)\bigr) = (u(1),\dots,u(T)),
\]
where each $u(t)\in\mathbb{R}^D$ satisfies 
\[
    |u(t)_i - f_i(x(t))| \le \epsilon \quad \text{for all } i=1,\dots,D.
\]
By Assumption~\ref{assump:approximation}, such an approximation can be 
implemented by a two-layer FFN with parameter count 
$O(1/\epsilon^{\gamma})$.  

\paragraph{Step 2: Approximating the minimization.}  
Let $F_2':\mathbb{R}^{D\times T}\to\mathbb{R}^D$ be defined by
\[
    F_2'(u(1),\dots,u(T)) = (u_1,\dots,u_D),
    \qquad 
    u_i = \min_{t\in S_i} u(t)_i.
\]
By Lemma~\ref{lem:relu-max} (which works the same for taking minimum), there exists a three-layer ReLU network with 
$O(1/\epsilon)$ parameters that $\epsilon$-approximates $F_2'$.  
We denote this approximation by $F_2$.  

\paragraph{Step 3: Approximating the outer function.}  
Finally, let $F_3:\mathbb{R}^D\to\mathbb{R}$ be a two-layer FFN that 
$\epsilon$-approximates $F_0$, with parameter count $O(1/\epsilon^\gamma)$.  

\paragraph{Composition.}  
Since $F_0$ is $C^1$ on a compact domain, it is Lipschitz with constant $L$, 
and the $\min$ operator is 1-Lipschitz. Therefore, the composed network
\[
    F_3 \circ F_2 \circ F_1
\]
provides an $L\epsilon$-approximation of the target function, with total 
parameter count 
\[
    O(1/\epsilon^{\gamma+1}).
\]
According to lemma~\ref{lem:stacking}, $F_3 \circ F_2 \circ F_1$ can be written equivalently as a five-layer FFN. 

\end{proof}
\subsection{Proof of Corollary~\ref{cor:uniqueness}}
\paragraph{Proof Sketch.} It is a direct corollary of Theorem~\ref{thm2:p1} and \ref{thm2:p2}.
\begin{proof}
    Suppose $D_1 < D_2$. Let $M_0$ be the minimal positive integer such that $\mathcal{H}(D_1,2,d,T,M_0)$ $\epsilon$-approximates $H$. Then with representation $(\{f_i, S_i\}_{i=1}^{D_1}, F_0)$, Theorem~\ref{thm2:p1} suggests that there exists a positive $C_{d, D_1, T}$ such that 
    \begin{equation}
        M_0 \le \frac{C_{d, D_1, T}}{\epsilon^\gamma}
    \end{equation}
    With representation $(\{\tilde{f}_i, \tilde{S}_i\}_{i=1}^{D_2}, \tilde{F}_0)$, Theorem~\ref{thm2:p2} suggests that there exists a positive $C_{d, D_2, T}$ such that 
    \begin{equation}
        M_0 \ge \frac{C_{d, D_2, T}}{\epsilon^k} \text{ for } k=\frac{(\frac{1}{4} T-D_1 -D_2 +1)}{3D_1+1}-1
    \end{equation}
    As $f_i$ and $F_0$ are at least $C^1$ smooth, we have $\gamma \le \max(D_1, D_2)$. Thus with $D_1^2 + D_2^2 \le \frac{1}{50}T$, we have $k > \gamma$. Then there exist $\epsilon>0$ such that 
    \begin{equation}
        \frac{C_{d, D_2, T}}{\epsilon^k} > \frac{C_{d, D_1, T}}{\epsilon^\gamma}
    \end{equation}
    This leads to a contradiction. Thus $D_1 = D_2$.
    
\end{proof}

\newpage
\section{Experiment Details}

\subsection{Details for Experiment 1}
\label{app:exp1-details}
\subsubsection{Experimental details for Section~\ref{experiment1}}
\label{app:exp1-training-details}

\paragraph{Data generation.}  
The intrinsic dimension of the synthetic task is $D=4$. For each sequence length 
$T \in \{8,16,32,64,128\}$ we generate $8000$ training and $2000$ validation 
examples. The inputs are i.i.d.\ Gaussian samples $x(t) \sim \mathcal{N}(0,I_4)$.  

\paragraph{Model architecture.}  
Each input vector $x(t)$ is first mapped to $\mathbb{R}^{8h}$ by a two-layer 
feed-forward network with hidden dimension $N$ and ReLU activations, ensuring 
a per-head embedding dimension of $8$. A trainable classification token $c_0$ 
is appended, and no positional encoding is used since the task is permutation 
invariant. The sequence is processed by a single-layer multi-head attention 
block without residual connections or layer normalization, consistent with the 
theoretical setting. The output is concatenated and passed through a two-layer 
GELU-activated feed-forward network with hidden dimension $N$, yielding the 
final scalar prediction. The fixed hidden size ensures comparability of 
parameter counts across different $h$.  

\paragraph{Training protocol.}  
Each configuration $(h,T)$ is trained separately under multiple random seeds. 
To reduce the effect of optimization variance, we report the \emph{minimal 
validation error} achieved across seeds. This choice isolates expressivity 
limitations of the architecture from randomness in training dynamics.  

\paragraph{Evaluation metric.}  
We adopt the normalized mean squared error (NMSE), defined as mean squared error 
divided by the variance of the targets. As $T$ increases, maxima of Gaussian 
samples concentrate, shrinking target variance and making trivial predictors 
appear competitive under raw MSE. (An intuition is that suppose the target output be $Y_T = \max_{1 \le t \le T} x_t$ with input tokens $x_t \sim \mathcal{N}(0,1)$ independently, then $\operatorname{Var}(Y_T)$ \emph{decreases} as $T$ increases, because $Y_T$ concentrates more tightly around its growing mean.)Normalization by variance corrects this effect 
and ensures comparability across lengths. NMSE is also equivalent to $1-R^2$, 
where $R^2$ is the standard coefficient of determination.  

\paragraph{Variance across seeds.}  
While mean performance across seeds is also informative, reporting the minimal 
validation NMSE highlights the best achievable accuracy for a given 
architecture. This emphasizes limitations due to model capacity rather than 
training noise. Tables showing seed variance are included for completeness (Table~\ref{table:synthetic-variance}).

\textit{Remark.}  \label{observation:reverse}
When $h \geq D =4$, we also observe that the validation NMSE first decreases rapidly 
and then increases slowly as $T$ grows. For shorter sequences, the model with 
enough heads can either capture the pattern through attention 
(Theorem~\ref{thm2:p1}) or rely on a memorization-based strategy with the 
feed-forward network (Theorem~\ref{thm2:p3}). Both approaches generalize 
reasonably well, but the memorization-based one does so less effectively. For 
longer sequences, memorization becomes infeasible and the model relies on 
attention, which generalizes better; however, longer sequences may also be more 
sensitive to parameterization, and the observed curve likely reflects a tradeoff 
between these effects. See Figure~\ref{fig: synthetic remark} in Appendix.

\subsubsection{Figures and Tables for Synthetic Experiment~\ref{experiment1}}
\begin{figure*}[h]
    \centering
    \includegraphics[width=0.6\textwidth]{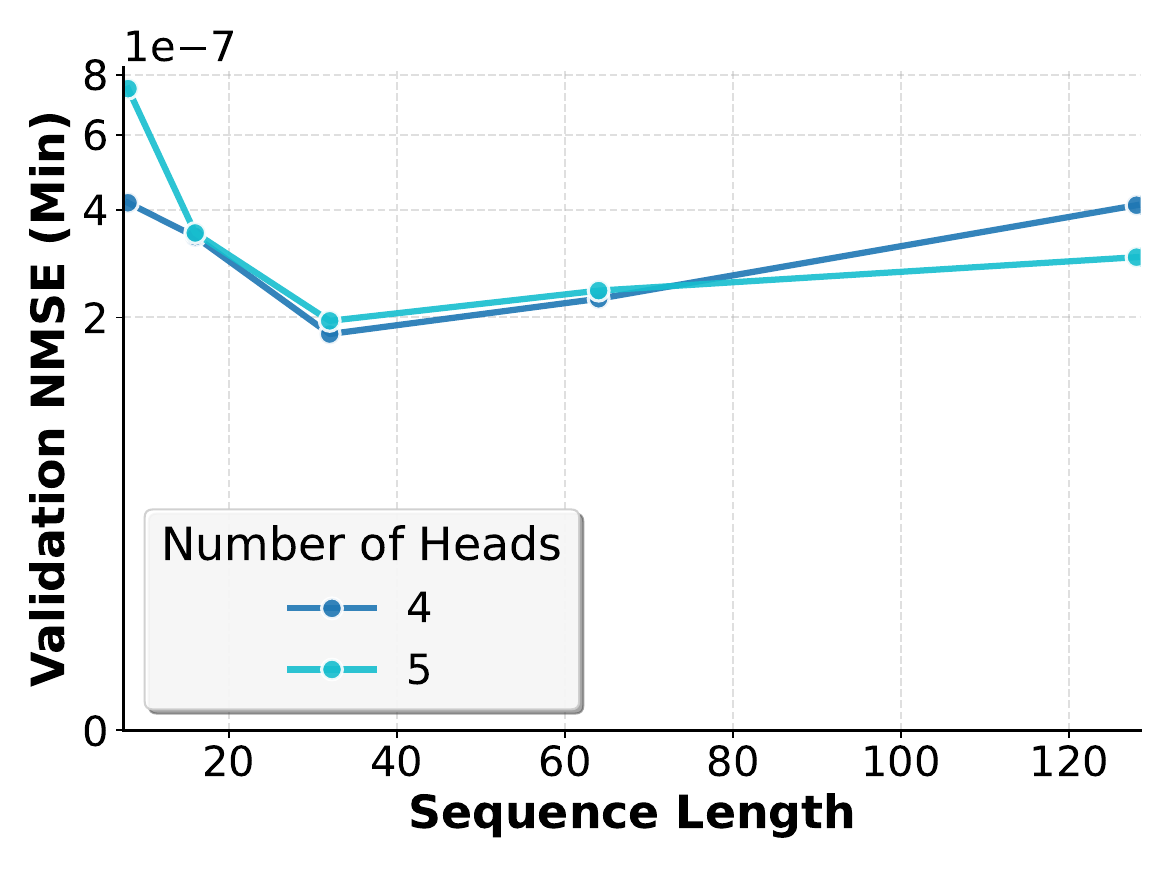}
    \caption{A zoom in plot of Figure\ref{fig:synthetic-nmse}, which shows that when the number of head is enough, the loss first decreases and then increases, as explained in the remark~\ref{observation:reverse}}
    
    \label{fig: synthetic remark}
\end{figure*}

\begin{figure*}[h]
    \centering

        \includegraphics[width=\linewidth]{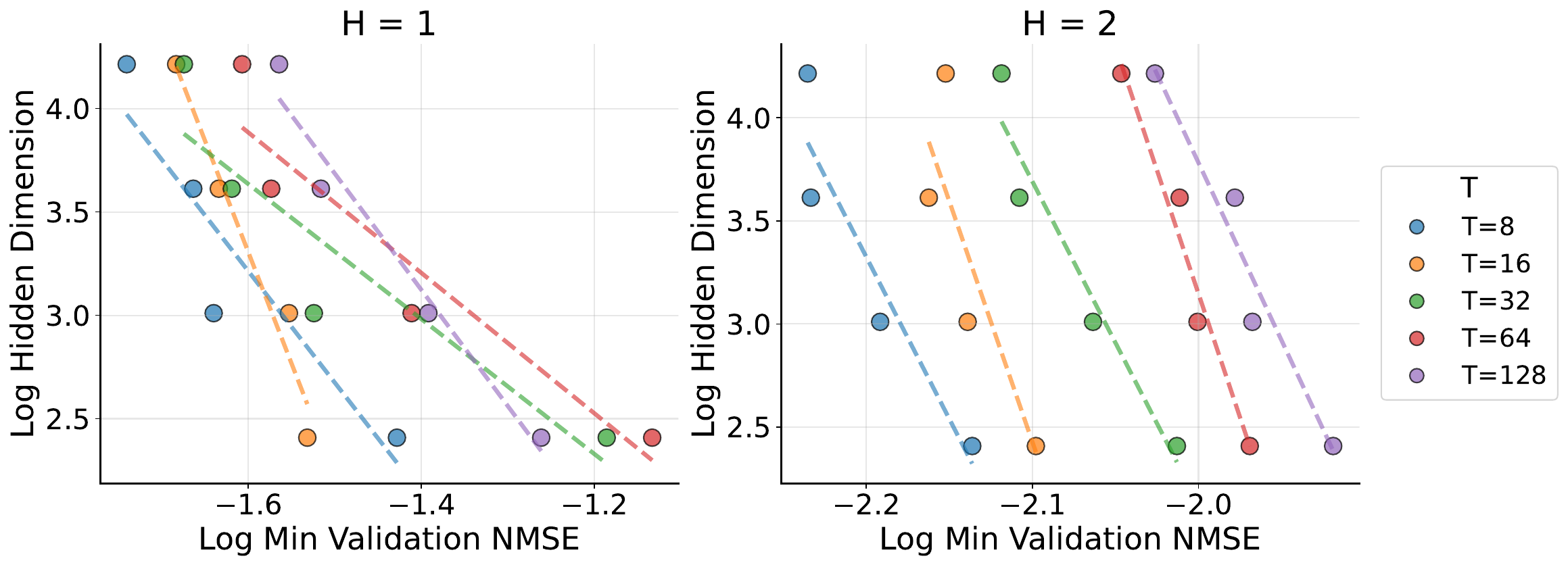}
        \caption{ Additional plot of \ref{fig:synthetic-nmse-b} for $H=1$ and $H=2$.
}
         \label{fig:synthetic-nmse-appendix}

\end{figure*}

\begin{table}[H]
\centering
\resizebox{\textwidth}{!}{%
\begin{tabular}{c|ccccc}
\hline
\textbf{Heads}& \textbf{T=8} & \textbf{T=16} & \textbf{T=32} & \textbf{T=64} & \textbf{T=128} \\
\hline
1 & $7.01\times10^{-2} \pm 5.99\times10^{-2}$ & $1.09\times10^{-1} \pm 9.93\times10^{-2}$ & $1.10\times10^{-1} \pm 9.36\times10^{-2}$ & $1.14\times10^{-1} \pm 8.53\times10^{-2}$ & $1.45\times10^{-1} \pm 1.05\times10^{-1}$ \\
2 & $7.31\times10^{-3} \pm 4.75\times10^{-4}$ & $8.41\times10^{-3} \pm 7.97\times10^{-4}$ & $9.42\times10^{-3} \pm 6.42\times10^{-4}$ & $1.31\times10^{-2} \pm 1.16\times10^{-2}$ & $1.47\times10^{-2} \pm 1.21\times10^{-2}$ \\
3 & $6.94\times10^{-4} \pm 2.90\times10^{-4}$ & $6.40\times10^{-4} \pm 3.87\times10^{-4}$ & $9.09\times10^{-4} \pm 4.31\times10^{-4}$ & $1.31\times10^{-3} \pm 5.10\times10^{-4}$ & $1.58\times10^{-3} \pm 5.21\times10^{-4}$ \\
4 & $6.10\times10^{-5} \pm 1.52\times10^{-4}$ & $4.36\times10^{-5} \pm 1.93\times10^{-4}$ & $4.80\times10^{-5} \pm 2.30\times10^{-4}$ & $8.75\times10^{-6} \pm 5.58\times10^{-5}$ & $5.23\times10^{-6} \pm 5.67\times10^{-6}$ \\
5 & $3.35\times10^{-5} \pm 5.84\times10^{-5}$ & $1.10\times10^{-5} \pm 2.36\times10^{-5}$ & $4.91\times10^{-6} \pm 6.32\times10^{-6}$ & $4.19\times10^{-6} \pm 8.39\times10^{-6}$ & $3.99\times10^{-6} \pm 4.29\times10^{-6}$ \\
\hline
\end{tabular}
}
\caption{Error bar for synthetic dataset. NMSE(Mean ± Standard Deviation) for different sequence lengths $T$ and number of heads.}
\label{table:synthetic-variance}
\end{table}

\begin{table}[H]
\centering
\resizebox{\textwidth}{!}{%
\begin{tabular}{c|ccccc}
\hline
\textbf{Heads} & \textbf{T=8} & \textbf{T=16} & \textbf{T=32} & \textbf{T=64} & \textbf{T=128} \\
\hline
1 & $1.75\times10^{-2}$ & $1.98\times10^{-2}$ & $2.06\times10^{-2}$ & $2.54\times10^{-2}$ & $3.03\times10^{-2}$ \\
2 & $7.17\times10^{-3}$ & $7.39\times10^{-3}$ & $7.82\times10^{-3}$ & $8.57\times10^{-3}$ & $1.02\times10^{-2}$ \\
3 & $2.11\times10^{-4}$ & $2.17\times10^{-4}$ & $2.73\times10^{-4}$ & $3.71\times10^{-4}$ & $4.77\times10^{-4}$ \\
4 & $1.32\times10^{-6}$ & $5.59\times10^{-7}$ & $3.40\times10^{-7}$ & $3.46\times10^{-7}$ & $5.70\times10^{-7}$ \\
5 & $2.19\times10^{-6}$ & $4.33\times10^{-7}$ & $3.22\times10^{-7}$ & $2.73\times10^{-7}$ & $2.66\times10^{-7}$ \\
\hline
\end{tabular}
}
\caption{Validation NMSE under fixed total embedding dimension $E = nh = 32$.}
\label{tab:heads-fixed-E}
\end{table}

\begin{table}[H]
\centering
\resizebox{\textwidth}{!}{%
\begin{tabular}{c|ccccc}
\hline
\textbf{Heads} & \textbf{T=8} & \textbf{T=16} & \textbf{T=32} & \textbf{T=64} & \textbf{T=128} \\
\hline
1 & $1.38\times10^{-2}$ & $1.63\times10^{-2}$ & $1.84\times10^{-2}$ & $2.17\times10^{-2}$ & $2.31\times10^{-2}$ \\
2 & $1.09\times10^{-3}$ & $7.08\times10^{-4}$ & $7.24\times10^{-4}$ & $7.76\times10^{-4}$ & $1.11\times10^{-3}$ \\
3 & $4.18\times10^{-7}$ & $1.72\times10^{-7}$ & $1.17\times10^{-7}$ & $3.58\times10^{-7}$ & $2.11\times10^{-7}$ \\
4 & $5.56\times10^{-7}$ & $1.22\times10^{-7}$ & $6.89\times10^{-8}$ & $1.85\times10^{-7}$ & $3.48\times10^{-7}$ \\
\hline
\end{tabular}
}
\caption{Approximation error for the $D=3$ retrieval task under fixed total embedding dimension $E = nh$.}
\label{tab:D3-results}
\end{table}

\begin{table}[H]
\centering
\resizebox{\textwidth}{!}{%
\begin{tabular}{c|ccccc}
\hline
\textbf{Heads} & \textbf{T=8} & \textbf{T=16} & \textbf{T=32} & \textbf{T=64} & \textbf{T=128} \\
\hline
1 & $2.12\times10^{-4}$ & $1.85\times10^{-4}$ & $2.22\times10^{-4}$ & $3.13\times10^{-4}$ & $4.28\times10^{-4}$ \\
2 & $7.22\times10^{-6}$ & $2.69\times10^{-6}$ & $3.50\times10^{-6}$ & $3.07\times10^{-6}$ & $3.83\times10^{-6}$ \\
3 & $8.16\times10^{-6}$ & $1.83\times10^{-6}$ & $1.73\times10^{-6}$ & $3.86\times10^{-6}$ & $3.50\times10^{-6}$ \\
4 & $3.68\times10^{-6}$ & $1.87\times10^{-6}$ & $2.60\times10^{-6}$ & $4.32\times10^{-6}$ & $5.98\times10^{-6}$ \\
5 & $6.15\times10^{-6}$ & $3.02\times10^{-6}$ & $3.31\times10^{-6}$ & $3.78\times10^{-6}$ & $5.34\times10^{-6}$ \\
\hline
\end{tabular}
}
\caption{Two-layer transformer on the synthetic task ($D=4$, NoPE, NoLN, fixed total embedding dimension $E=nh=32$).}
\label{tab:2layer}
\end{table}

\subsection{MS MARCO Text Retrieval}\label{app:exp-msmarco}
\subsubsection{Experiment Details for MS MARCO (text retrieval) Experiment}
\paragraph{Dataset construction.}  
We construct retrieval-style datasets from the MS MARCO passage ranking collection 
\citep{bajaj2018.MSMARCOHuman}. Since the original dataset associates each query with only a 
few candidate passages, we enlarge the candidate set by mining hard negatives. 
Specifically, BM25 \citep{robertson2009.ProbabilisticRelevanceFramework} is used to mine local negatives 
and FAISS \citep{johnson2017.BillionscaleSimilaritySearch} similarity search 
to retrieve global negatives, reducing redundancy across queries. 
For each query, the sequence length $T$ is defined as the total number of candidates 
(one positive and $T-1$ negatives), with $T \in \{8,16,32,64\}$. 
We build datasets containing $28{,}000$ training queries and $2{,}000$ validation 
queries for each $T$.

\paragraph{Model and training setup.}  
We evaluate a two-layer Transformer encoder with per-head embedding dimension 
fixed at $32$, while varying the number of heads across 
$\{1,2,4,6,8,10,12,14,16\}$. Tokenization and input embeddings follow the BERT 
tokenizer and frozen BERT word, position, and segment embeddings 
\citep{devlin2019.BERTPretrainingDeep}, projected to the model hidden size 
$h = \text{heads} \times 32$. Only the projection and Transformer layers are trained. 
We report training top-1 accuracy, focusing on training performance since 
MS MARCO with BM25-mined negatives is particularly challenging for validation, and the difference can be seen in training metrics. Training MRR is also reported in Fig~\ref{fig:ms mrr}, with similar trend as training accuracy. 

\subsubsection{Figures and Tables for Experiment}
\begin{table}[h!]
\centering
\begin{tabular}{c|cccc}
\hline
\textbf{Heads} & \textbf{T=8} & \textbf{T=16} & \textbf{T=32} & \textbf{T=64} \\
\hline
1  & $0.597 \pm 0.003$ & $0.450 \pm 0.005$ & $0.303 \pm 0.003$ & $0.154 \pm 0.002$ \\
2  & $0.771 \pm 0.003$ & $0.647 \pm 0.003$ & $0.486 \pm 0.002$ & $0.286 \pm 0.002$ \\
4  & $0.956 \pm 0.002$ & $0.900 \pm 0.002$ & $0.793 \pm 0.002$ & $0.580 \pm 0.004$ \\
6  & $0.992 \pm 0.000$ & $0.977 \pm 0.001$ & $0.937 \pm 0.001$ & $0.814 \pm 0.002$ \\
8  & $0.998 \pm 0.000$ & $0.995 \pm 0.000$ & $0.983 \pm 0.001$ & $0.932 \pm 0.002$ \\
12 & $1.000 \pm 0.000$ & $0.999 \pm 0.000$ & $0.998 \pm 0.000$ & $0.991 \pm 0.001$ \\
16 & $1.000 \pm 0.000$ & $1.000 \pm 0.000$ & $0.999 \pm 0.000$ & $0.996 \pm 0.000$ \\
\hline
\end{tabular}
\caption{Error bar for MS Marco dataset. Accuracy (Mean ± Standard Deviation) for different sequence lengths $T$ and number of heads.}
\label{table:MS_Mean_var}
\end{table}

\begin{table}[h!]
\centering
\begin{tabular}{c|cccc}
\hline
\textbf{Heads} & \textbf{T=8} & \textbf{T=16} & \textbf{T=32} & \textbf{T=64} \\
\hline
1  & $0.5107 \pm 0.0069$ & $0.3917 \pm 0.0071$ & $0.2542 \pm 0.0049$ & $0.1257 \pm 0.0057$ \\
2  & $0.5221 \pm 0.0102$ & $0.4205 \pm 0.0056$ & $0.2712 \pm 0.0067$ & $0.1369 \pm 0.0038$ \\
4  & $0.5076 \pm 0.0112$ & $0.4048 \pm 0.0070$ & $0.2547 \pm 0.0093$ & $0.1139 \pm 0.0061$ \\
6  & $0.5153 \pm 0.0112$ & $0.3865 \pm 0.0103$ & $0.2397 \pm 0.0098$ & $0.1018 \pm 0.0057$ \\
8  & $0.5058 \pm 0.0082$ & $0.3801 \pm 0.0084$ & $0.2308 \pm 0.0068$ & $0.0983 \pm 0.0050$ \\
12 & $0.5184 \pm 0.0073$ & $0.3721 \pm 0.0107$ & $0.2219 \pm 0.0091$ & $0.0902 \pm 0.0054$ \\
16 & $0.5021 \pm 0.0123$ & $0.2816 \pm 0.0418$ & $0.2170 \pm 0.0075$ & $0.0878 \pm 0.0058$ \\
\hline
\end{tabular}
\caption{MS Marco Validation Accuracy (Mean ± Standard Deviation) for different sequence lengths $T$ and number of heads.}\label{tab:MS_val_mean_var}
\end{table}

\begin{figure*}[h]
    \centering

        \includegraphics[width=0.6\linewidth]{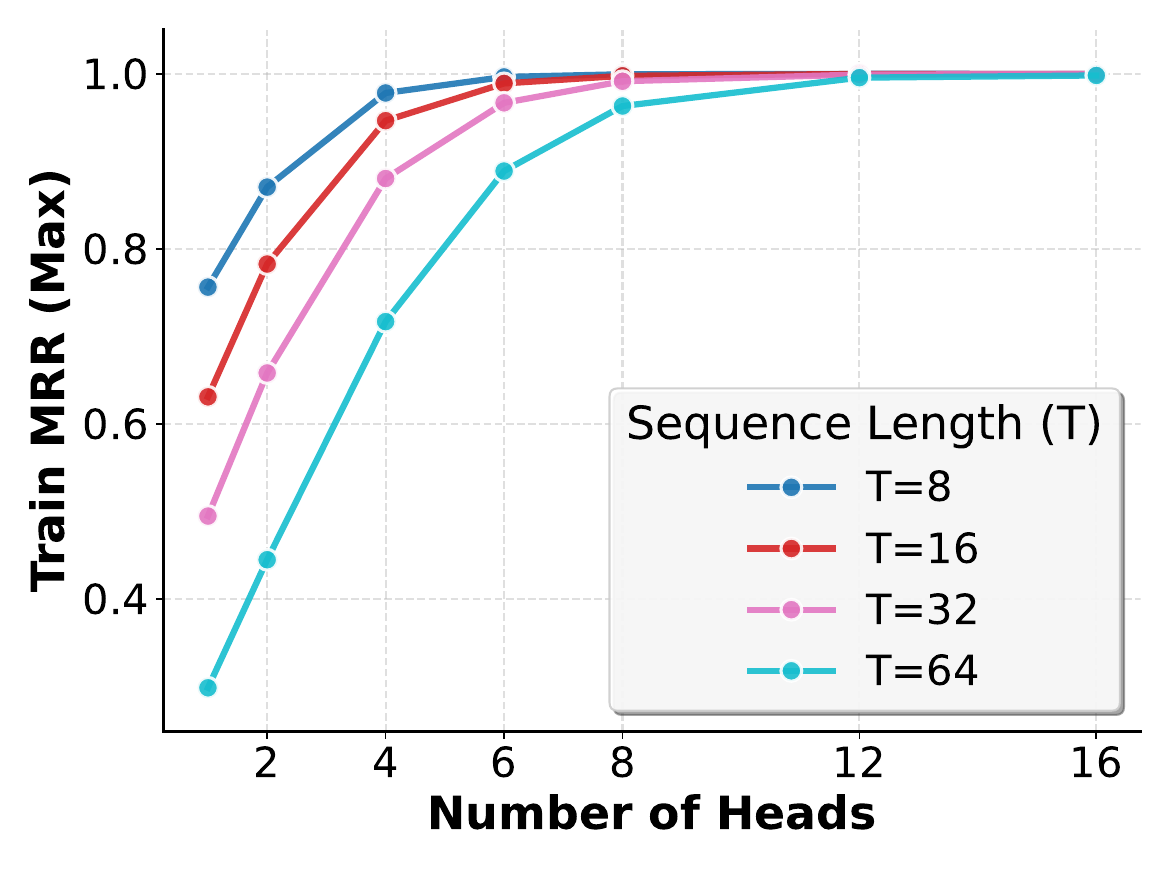}
        \caption{ Plot of training mrr for MS MARCO dataset.}
         \label{fig:ms mrr}

\end{figure*}

\subsection{CIFAR-10 Image Classification}\label{app:exp-cifar}

\subsubsection{Dataset Construction}
We create image classification datasets from the CIFAR-10 dataset using a padded preprocessing approach. The original CIFAR-10 images have dimensions of $32 \times 32$ pixels. To generate datasets with larger image sizes, we apply padding to achieve sizes in the set $\{32, 48, 64, 96, 128\}$. The original image is randomly positioned within the enlarged frame, with padding filled using the colors of the border pixels. An illustration is provided in Figure \ref{fig:cifar padded}. 
By apply this padding method we are creating tasks with increasing difficulty. The background is enlarged, making models need more effort to learn how to extract useful information. The random placement make sure the padded outside aera cannot be simply ignored by position encodings.

Each image is divided into non-overlapping patches of size $8 \times 8$ pixels, resulting in a sequence of patches for each image. For each image size, the sequence length $T$ is defined as the total number of patches plus one additional class token, with $T = \{17, 37, 65, 145, 257\}$. We adopt the standard CIFAR-10 data splits, which include $50,000$ training images and $10,000$ test images across $10$ classes.

\subsubsection{Model Training Setup}
We evaluate a Vision Transformer (ViT) with four layers and a per-head embedding dimension of $16$, while varying the number of attention heads in different configurations. Each image patch is embedded through a linear projection, and positional embeddings are added along with a learnable class token. No global convolutional embedding is used.

Input processing follows the standard ViT procedure, including patch embedding of size $8 \times 8$, positional encoding, and aggregation of the class token for final classification. The model is trained using the AdamW optimizer with cosine annealing learning rate scheduling. Standard architectural techniques, such as layer normalization, residual connections, and dropout, are employed for regularization.

\begin{figure*}[h]
    \centering

        \includegraphics[width=\linewidth]{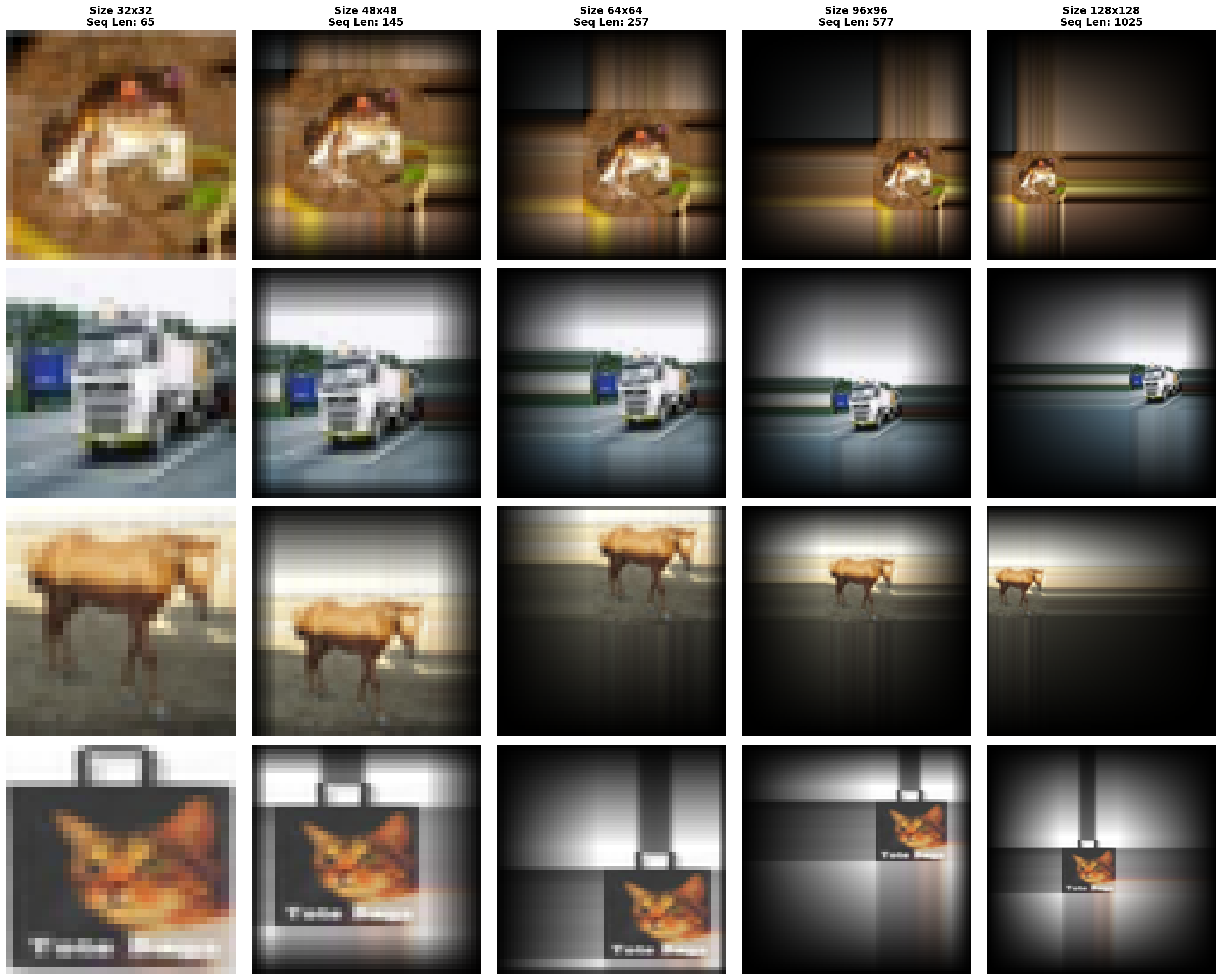}
        \caption{ Examples of the padded images from the dataset.}
         \label{fig:cifar padded}

\end{figure*}

\begin{table}[H]
\centering
\resizebox{\textwidth}{!}{%
\begin{tabular}{c|ccccc}
\hline
\textbf{Heads}& \textbf{Seq=65} & \textbf{Seq=145} & \textbf{Seq=257} & \textbf{Seq=577} & \textbf{Seq=1025} \\
\hline
1 & $4.78\times10^{1} \pm 4.50\times10^{-1}$ & $4.37\times10^{1} \pm 4.70\times10^{-1}$ & $4.20\times10^{1} \pm 5.00\times10^{-1}$ & $4.08\times10^{1} \pm 6.50\times10^{-1}$ & $4.00\times10^{1} \pm 7.20\times10^{-1}$ \\
2 & $5.97\times10^{1} \pm 4.50\times10^{-1}$ & $5.52\times10^{1} \pm 3.80\times10^{-1}$ & $5.34\times10^{1} \pm 4.40\times10^{-1}$ & $5.15\times10^{1} \pm 4.70\times10^{-1}$ & $5.08\times10^{1} \pm 7.40\times10^{-1}$ \\
4 & $7.55\times10^{1} \pm 2.10\times10^{-1}$ & $7.03\times10^{1} \pm 3.20\times10^{-1}$ & $6.85\times10^{1} \pm 7.70\times10^{-1}$ & $6.62\times10^{1} \pm 6.00\times10^{-1}$ & $6.58\times10^{1} \pm 7.20\times10^{-1}$ \\
8 & $9.51\times10^{1} \pm 1.50\times10^{-1}$ & $9.26\times10^{1} \pm 4.70\times10^{-1}$ & $9.14\times10^{1} \pm 6.20\times10^{-1}$ & $9.07\times10^{1} \pm 1.02\times10^{0}$ & $9.02\times10^{1} \pm 1.00\times10^{0}$ \\
10 & $9.81\times10^{1} \pm 5.00\times10^{-2}$ & $9.73\times10^{1} \pm 2.40\times10^{-1}$ & $9.67\times10^{1} \pm 4.60\times10^{-1}$ & $9.65\times10^{1} \pm 3.20\times10^{-1}$ & $9.67\times10^{1} \pm 7.30\times10^{-1}$ \\
11 & $9.88\times10^{1} \pm 1.20\times10^{-1}$ & $9.83\times10^{1} \pm 1.20\times10^{-1}$ & $9.81\times10^{1} \pm 2.40\times10^{-1}$ & $9.77\times10^{1} \pm 2.10\times10^{-1}$ & $9.79\times10^{1} \pm 2.20\times10^{-1}$ \\
12 & $9.92\times10^{1} \pm 3.00\times10^{-2}$ & $9.89\times10^{1} \pm 1.60\times10^{-1}$ & $9.86\times10^{1} \pm 2.40\times10^{-1}$ & $9.86\times10^{1} \pm 1.90\times10^{-1}$ & $9.86\times10^{1} \pm 2.90\times10^{-1}$ \\
13 & $9.94\times10^{1} \pm 6.00\times10^{-2}$ & $9.93\times10^{1} \pm 6.00\times10^{-2}$ & $9.91\times10^{1} \pm 1.10\times10^{-1}$ & $9.90\times10^{1} \pm 2.00\times10^{-1}$ & $9.91\times10^{1} \pm 2.50\times10^{-1}$ \\
14 & $9.96\times10^{1} \pm 3.00\times10^{-2}$ & $9.94\times10^{1} \pm 9.00\times10^{-2}$ & $9.93\times10^{1} \pm 1.00\times10^{-1}$ & $9.93\times10^{1} \pm 1.70\times10^{-1}$ & $9.93\times10^{1} \pm 2.30\times10^{-1}$ \\
16 & $9.97\times10^{1} \pm 3.00\times10^{-2}$ & $9.96\times10^{1} \pm 2.00\times10^{-2}$ & $9.95\times10^{1} \pm 7.00\times10^{-2}$ & $9.96\times10^{1} \pm 1.40\times10^{-1}$ & $9.96\times10^{1} \pm 1.70\times10^{-1}$ \\
20 & $9.98\times10^{1} \pm 1.00\times10^{-2}$ & $9.97\times10^{1} \pm 5.00\times10^{-2}$ & $9.97\times10^{1} \pm 8.00\times10^{-2}$ & $9.98\times10^{1} \pm 6.00\times10^{-2}$ & $9.99\times10^{1} \pm 7.00\times10^{-2}$ \\
24 & $9.99\times10^{1} \pm 2.00\times10^{-2}$ & $9.98\times10^{1} \pm 2.00\times10^{-2}$ & $9.98\times10^{1} \pm 4.00\times10^{-2}$ & $9.99\times10^{1} \pm 4.00\times10^{-2}$ & $9.99\times10^{1} \pm 5.00\times10^{-2}$ \\
\hline
\end{tabular}
}
\caption{Error bar for Image task. Accuracy (Mean ± Standard Deviation) for different sequence lengths and number of heads.}
\label{table: image}
\end{table}

\begin{table}[h!]
\centering
\begin{tabular}{c|ccccc}
\hline
\textbf{Heads} & \textbf{T=65} & \textbf{T=145} & \textbf{T=257} & \textbf{T=577} & \textbf{T=1025} \\
\hline
1 & $50.50 \pm 0.44$ & $45.85 \pm 0.58$ & $43.53 \pm 0.58$ & $41.43 \pm 0.97$ & $39.95 \pm 0.75$ \\
2 & $60.01 \pm 0.57$ & $55.01 \pm 0.25$ & $53.12 \pm 0.69$ & $49.95 \pm 0.80$ & $48.53 \pm 0.83$ \\
4 & $67.98 \pm 0.43$ & $63.49 \pm 0.70$ & $61.42 \pm 0.61$ & $57.19 \pm 0.54$ & $55.31 \pm 0.85$ \\
8 & $69.65 \pm 0.55$ & $66.06 \pm 0.53$ & $62.43 \pm 0.95$ & $57.58 \pm 0.59$ & $56.18 \pm 1.14$ \\
10 & $69.70 \pm 0.24$ & $65.64 \pm 0.37$ & $62.45 \pm 0.49$ & $57.21 \pm 1.03$ & $54.44 \pm 1.49$ \\
11 & $69.84 \pm 0.36$ & $65.26 \pm 0.43$ & $61.97 \pm 0.29$ & $56.95 \pm 0.63$ & $53.77 \pm 2.91$ \\
12 & $69.66 \pm 0.39$ & $65.79 \pm 0.56$ & $62.63 \pm 0.36$ & $56.57 \pm 0.66$ & $53.17 \pm 1.06$ \\
13 & $69.72 \pm 0.18$ & $65.30 \pm 0.59$ & $61.66 \pm 0.58$ & $54.81 \pm 2.23$ & $52.90 \pm 1.45$ \\
14 & $69.49 \pm 0.48$ & $65.25 \pm 0.48$ & $61.32 \pm 0.95$ & $54.01 \pm 2.13$ & $50.42 \pm 2.48$ \\
16 & $69.69 \pm 0.33$ & $64.24 \pm 0.33$ & $59.29 \pm 1.12$ & $49.68 \pm 1.39$ & $48.51 \pm 2.59$ \\
20 & $67.99 \pm 0.35$ & $61.49 \pm 1.07$ & $55.25 \pm 2.12$ & $48.65 \pm 0.85$ & $46.89 \pm 1.07$ \\
24 & $65.12 \pm 0.68$ & $56.80 \pm 1.09$ & $52.14 \pm 0.56$ & $48.74 \pm 0.60$ & $48.39 \pm 0.76$ \\
\hline
\end{tabular}
\caption{Error bar for Image task. Validation Accuracy (Mean ± Standard Deviation) for different sequence lengths and number of heads.}
\label{table:vit_val_accuracy}
\end{table}


\begin{table}[htbp] 
\centering
\caption{Hyperparameter settings of popular transformer models. Only d (embedding dimension), L (number of layers), and H (number of attention heads) are shown for brevity.}
\label{table: head count}
\begin{tabular}{|c|l|c|c|c|}
\hline
\textbf{H} & \textbf{Model} & \textbf{d} & \textbf{L} & \textbf{Year} \\
\hline
8 & Attention is all you need & 512 & 6 & 2017 \\
8 & Gemma 2B & 2,048 & 18 & 2024 \\
12 & GPT & 768 & 12 & 2018 \\
16 & BERT-Large & 1,024 & 24 & 2019 \\
16 & ViT-Huge & 1,280 & 32 & 2021 \\
16 & Gemma 7B & 3,072 & 28 & 2024 \\
28 & Turing-NLG & 4,256 & 78 & 2020 \\
32 & LLaMA-7B & 4,096 & 32 & 2023 \\
32 & Baichuan 2-7B & 4,096 & 32 & 2023 \\
32 & Mistral 7B & 4,096 & 32 & 2023 \\
32 & Yi-6B & 4,096 & 32 & 2023 \\
32 & LLaMA 3-8B & 4,096 & 32 & 2024 \\
32 & Mixtral 8x7B & 4,096 & 32 & 2024 \\
40 & LLaMA-13B & 5,120 & 40 & 2023 \\
40 & Baichuan 2-13B & 5,120 & 40 & 2023 \\
56 & Yi-34B & 7,168 & 60 & 2023 \\
64 & LLaMA-65B & 8,192 & 80 & 2023 \\
64 & Llama-2-70B & 8,192 & 80 & 2023 \\
64 & LLaMA 3-70B & 8,192 & 80 & 2024 \\
96 & GPT-3 & 12,288 & 96 & 2020 \\
96 & Jurassic-1 & 13,824 & 76 & 2021 \\
128 & MT-NLG & 20,480 & 105 & 2021 \\
128 & LaMDA & 8,192 & 64 & 2022 \\
128 & LLaMA 3.1-405B & 16,384 & 126 & 2024 \\
128 & DeepSeek-V2 & 5,120 & 60 & 2024 \\
\hline
\end{tabular}
\end{table}

\section{Large Language Model Usage} 
Large language models were used only for linguistic refinement (e.g., polishing sentences and checking grammar). 
The core ideas, theoretical results, experimental design, and analyses presented in this paper were entirely developed by the authors without assistance from large language models.

\end{document}